\documentclass[preprint]{article}

\usepackage[nonatbib]{neurips_2026}
\usepackage[utf8]{inputenc}
\usepackage[T1]{fontenc}
\usepackage{url}
\usepackage{booktabs}
\usepackage{amsfonts}
\usepackage{nicefrac}
\usepackage{xcolor}
\usepackage{mathtools}
\usepackage{graphicx}
\usepackage{algorithm}
\usepackage{algpseudocode}
\usepackage{enumitem}
\usepackage{subcaption}

\usepackage[round]{natbib}

\usepackage{amsmath}
\usepackage{amssymb}
\usepackage{amsthm}
\usepackage{graphicx}
\usepackage{booktabs}
\usepackage{array}
\usepackage{multirow}
\usepackage{float}
\usepackage[colorlinks=true,linkcolor=blue!60!black,citecolor=blue!60!black,urlcolor=blue!60!black,pdfauthor={},pdftitle={},pdfsubject={},pdfkeywords={},pdfcreator={},pdfproducer={}]{hyperref}
\usepackage[capitalize,noabbrev]{cleveref}

\newtheorem{theorem}{Theorem}
\newtheorem{proposition}[theorem]{Proposition}
\newtheorem{assumption}{Assumption}
\theoremstyle{definition}
\newtheorem{definition}{Definition}
\crefname{assumption}{Assumption}{Assumptions}
\crefname{proposition}{Proposition}{Propositions}
\crefname{definition}{Definition}{Definitions}
\AddToHook{env/proposition/begin}{\crefalias{theorem}{proposition}}

\newcommand{\E}{\mathbb{E}}
\newcommand{\Var}{\operatorname{Var}}
\newcommand{\Prob}{\mathbb{P}}

\newcommand{\D}{\mathcal{D}}

\newcommand{\Dclean}{\mathcal{D}_{\mathrm{clean}}}
\newcommand{\Dcorr}{\mathcal{D}_{\mathrm{corr}}}
\newcommand{\Dbdry}{\mathcal{D}_{\mathrm{bdry}}}
\newcommand{\Dbulk}{\mathcal{D}_{\mathrm{corr}}^{\mathrm{bulk}}}
\newcommand{\Dtail}{\mathcal{D}_{\mathrm{corr}}^{\mathrm{tail}}}
\newcommand{\atrim}{\alpha_{\mathrm{trim}}}
\newcommand{\rK}{r_{K}}
\newcommand{\rankvar}{\widehat{\sigma}^{2}}
\newcommand{\rankbar}{\overline{\sigma}^{2}}
\newcommand{\drisShort}{\textsc{DR-IS}}
\newcommand{\drisLong}{Disagreement-Regularized Importance Sampling}

\title{Rank-Variance Scores for Importance Sampling and Data Pruning under Label Contamination}

\author{%
  Csongor Horváth\thanks{Corresponding author: \texttt{csongor.horvath@it.uu.se}} \textsuperscript{ ,}\thanks{%
    \protect\begin{tabular}[t]{@{}l@{}}
      Department of Information Technology, Uppsala University, Uppsala, Sweden \\
      Science for Life Laboratory, Uppsala University, Uppsala, Sweden
    \end{tabular}} \\
  \And
  Ida-Maria Sintorn\footnotemark[2] \\
  \And
  Prashant Singh\footnotemark[2] \\
}

\begin{document}
\maketitle

\begin{abstract}
Variance-optimal importance sampling and magnitude-based pruning prioritize high-loss or large-gradient examples, which under label contamination are often the corrupted ones. Under Huber's $\varepsilon$-contamination model, we characterize this failure exactly: a score sampler's share of visits to corrupted examples is set by their mean score relative to the dataset mean, and a pruned subset's contamination is tightly bounded by the score's corruption-detection AUROC; at every loss-based score we measure, pruning is provably no cleaner than random selection. A useful score must therefore be non-monotone in the loss. We build such a score from the loss ranks of $K$ independently trained proxies: a mislabelled example ranks consistently hardest, a near-boundary example unstably. Its total rank variance (TRV) splits into a between-proxy part, Disagreement-Regularized Importance Sampling (DR-IS), and a within-proxy part; the split predicts that DR-IS is the better sampler and TRV the better pruning rule, both checkable without labels before training. Experiments on CIFAR-10/100 under label noise confirm both roles: the rules substantially outperform random selection where magnitude-based rules collapse, at a proxy cost of 2--4\% of training.
\end{abstract}

\section{Introduction}
\label{sec:intro}

Training classification models on all available examples uniformly may be infeasible in practice, either because datasets keep growing, which motivates keeping only a fraction $\alpha$ of the data, or because examples carry unequal information, which motivates weighting them. In both cases standard methods use magnitude scores, such as the loss or the gradient norm \citep{zhao2015stochastic,needell2016stochastic,loshchilov2015online,katharopoulos2018not,paul2021deep}, motivated by the variance reduction that importance sampling (IS) gives for stochastic gradient estimates. These methods do not consider that labels may be corrupted, which is often the case in real datasets \citep{wei2022learning}. Corrupted examples often correspond to the largest losses and gradients, which makes these methods favor them \citep{arpit2017closer, han2018coteaching, pleiss2020identifying}. With only $10\%$ uniform label noise, keeping the quarter of CIFAR-10 with the largest EL2N score drops test accuracy from $81.8\%$ (a random quarter) to $27.5\%$, and, when the $25\%$ highest-gradient examples are mislabelled, gradient-norm IS sends $83\%$ of its visits to them and ends at chance (\cref{tab:cifar_pruning,tab:online_is}).

\paragraph{Why magnitude-based selection fails.} Variance-optimal IS minimizes the variance of the \emph{empirical} risk gradient, which is the wrong target when the empirical distribution is contaminated. We make this precise under Huber's $\varepsilon$-contamination model for loss-aligned corruption, in which wrong labels mostly have larger losses than clean ones under a briefly trained proxy (\Cref{def:loss_aligned}). In this setting, magnitude-based selection provably fails both as a sampler and as a pruning rule (\Cref{prop:suboptimality}). For sampling, the ratio $x$ of the corrupted examples' mean score to the dataset mean fixes their share of visits, so any unbiased sampler whose score is higher on average on the corrupted examples visits them at least as often as uniform sampling. For pruning, the score's corruption-detection AUROC bounds the contamination of the kept subset from both sides. Above an AUROC of $1-\alpha(1-\alpha)$ the subset is no cleaner than a random one for every $\varepsilon$. In our main setting, every loss-based score we measured is above this threshold and keeps subsets that are $77$--$87\%$ corrupted, against $25\%$ for random selection.

\paragraph{Rank disagreement.} Under possible contamination, a useful score should therefore be non-monotone in the loss. We build one from a \emph{memorization asymmetry}: briefly trained proxies from independent initializations rank clean examples near the decision boundary very differently, while a mislabelled example is hard for all of them, so its rank barely varies. \drisLong{} (\drisShort{}) scores each example by the variance of its loss rank across $K$ proxies. Its pruning counterpart, the total rank variance (TRV) $\overline\rho_i(1-\overline\rho_i)$, depends only on the mean rank $\overline\rho_i$, folded so that both the hardest and the easiest examples score low. \Cref{prop:ranks,prop:suboptimality} predict which score suits which use: for sampling, what matters is $x$, the corrupted examples' relative mean score, which is lower for \drisShort{}; for pruning, it is the AUROC, which is lower for TRV. For large ensembles, both scores separate the corrupted bulk from the clean boundary at rate $O(\sqrt{\log(N/\delta)/K})$ (\Cref{thm:main}).

\paragraph{Results.} With $K=3$ proxies that add $2$--$4\%$ to the training cost, on CIFAR-10/100 under uniform and high-norm label noise:
\begin{itemize}[leftmargin=*,topsep=1pt,itemsep=1pt]
    \item \textbf{Static pruning} ($\alpha=0.25$): TRV improves on random pruning by $3.8$--$11.3$\,pp in every corrupted setting and \drisShort{} by up to $9.4$\,pp, while EL2N and  Consensus-loss collapse to $1$--$46\%$ accuracy on subsets that are up to $90\%$ corrupted.
    \item \textbf{Online IS}: \drisShort{} outperforms uniform SGD by $2.6$--$9.6$\,pp under noise and early stopping chosen on the test set by $1.6$--$2.9$\,pp, and reduces the fraction of corrupted labels the network fits by $83$--$87\%$; it incurs a cost of $1.2$\,pp on clean data.
\end{itemize}

\paragraph{Contributions.}
\begin{enumerate}[leftmargin=*,topsep=2pt,itemsep=1pt]
    \item \textbf{Failure of magnitude-based selection} (\Cref{prop:suboptimality}): an exact visit identity for unbiased samplers and attained, $\varepsilon$-free AUROC bounds on the contamination of pruned subsets, above the random-selection threshold.
    \item \textbf{Rank-disagreement scores and their analysis} (\Cref{prop:ranks,thm:main,prop:online}): \drisShort{} and TRV, the law-of-total-variance decomposition that predicts their roles, large-ensemble separation and contamination bounds, and an unbiased online sampler with bounded variance inflation.
    \item \textbf{Pre-deployment checks} (\Cref{sec:checks}): label-free estimates of $\varepsilon$, of the AUROC and of $x$, and a disagreement-histogram test of the assumptions.
    \item \textbf{Validation} (\Cref{sec:exp}): static and online experiments against magnitude, consensus, training-dynamics, and reference-model baselines.
\end{enumerate}

\paragraph{Scope.} Our claims hold under loss-aligned corruption. Standard uniform flips are loss-aligned on CIFAR-10/100, and magnitude rules collapse under them, as \Cref{prop:suboptimality} predicts; we also flip the highest-gradient examples (high-norm flips) (\Cref{prop:suboptimality}'s worst case), but the gains persist under other noise models. Real human noise on CIFAR-10N \citep{wei2022learning} is also loss-aligned: in a pre-registered test, EL2N collapses while TRV and \drisShort{}-online beat their baselines (\Cref{app:cifar10n}). On Food-101, whose labels are ambiguous rather than contradicting, the rank scores stay within $1.3$\,pp of random pruning while magnitude and margin rules lose $4.9$--$8.4$\,pp. The scores change only which examples SGD sees, so they combine with loss- or label-correcting methods; full DivideMix matches \drisShort{}-online at $4.8\times$ its wall-clock.

\section{Related Work}
\label{sec:related}
\paragraph{Magnitude-based selection.} Sampling proportional to the gradient norm minimizes the variance of the SGD update \citep{strohmer2009randomized,zhao2015stochastic,needell2016stochastic,csiba2018importance}. Deep-learning variants replace the gradient norm by the loss or an error surrogate \citep{katharopoulos2018not,jiang2019accelerating,alain2015variance,johnson2018training}, and the same idea drives online batch selection \citep{loshchilov2015online}, prioritized replay \citep{schaul2016prioritized} and EL2N/GraNd pruning \citep{paul2021deep,coleman2020selection}. RHO-LOSS \citep{mindermann2022prioritized} subtracts the loss of a reference model trained on clean holdout data. Bounds relating the AUC to the composition of the top of a ranking are classical in bipartite ranking \citep{clemencon2007ranking,boyd2012accuracy}. \Cref{prop:suboptimality} gives an attained  two-sided form with $\varepsilon$-free thresholds and adds a visit identity for samplers.

\paragraph{Training dynamics and disagreement.} AUM \citep{pleiss2020identifying}, forgetting events \citep{toneva2019empirical}, dataset cartography \citep{swayamdipta2020dataset} and Active Bias \citep{chang2017active} score examples from the training dynamics of a single run. Disagreement within a committee is a standard uncertainty signal in active learning \citep{seung1992qbc,gal2017deep,lakshminarayanan2017simple}. Agreement across models has long been used to flag mislabelled examples, by ensemble consensus filtering \citep{brodley1999identifying} and ensemble consensus during training \citep{lee2020robust}, and disagreement between two networks drives the updates of Decoupling and Co-teaching+ \citep{malach2017decoupling,yu2019coteachingplus}. \drisShort{} instead uses the variability across independently trained proxies, in ranks rather than probabilities, as a selection score for SGD, and we analyze what this score does to visits and to subset contamination. Moderate and coverage-centric coresets \citep{xia2023moderate,zheng2023coverage} keep examples of intermediate difficulty. TRV is a rule of this kind in rank space, with contamination bounds.

\paragraph{Noisy-label learning and robust statistics.} Co-teaching, JoCoR, confident learning, CRUST, Prune4ReL and DivideMix \citep{han2018coteaching,wei2020combating,northcutt2021confident,mirzasoleiman2020crust,park2023robust,li2020dividemix} change the objective, filter examples or relabel them. Our model is the label-only form of $\varepsilon$-contamination \citep{huber1964robust,diakonikolas2019recent}. Trimmed-loss estimators \citep{diakonikolas2019sever,shen2019learning} use the loss magnitude in the opposite direction, discarding the largest losses. Meta-reweighting \citep{ren2018learning,shu2019meta} and reference-model selection such as RHO-LOSS need trusted clean data. \drisShort{} does not, and it composes with pretrained representations: on Food-101 its proxies are linear probes on frozen features (\Cref{app:food101}).

\section{Magnitude Scores Fail}
\label{sec:setting}

A classifier $h_w$ is trained on $\D=\{(x_i,y_i)\}_{i=1}^N$ by minimizing the empirical risk $\widehat R=\frac1N\sum_if_i$, where $f_i(w)=\ell(h_w(x_i),y_i)$ is the loss of example $i$. For any $q>0$, sampling $i\sim q$ and stepping along $\omega_i\nabla f_i(w)$ with $\omega_i=(Nq_i)^{-1}$ is unbiased for $\nabla\widehat R(w)$. The variance of this empirical gradient estimate is minimized by $q^\star_i\propto\|\nabla f_i(w)\|$ \citep{zhao2015stochastic,needell2016stochastic}. Under contamination, $\widehat R$ no longer tracks the risk on clean data, and the examples with the largest gradients are often the corrupted ones.

\begin{assumption}[$\varepsilon$-contaminated labels]
\label{assu:contamination}
$\D=\Dclean\cup\Dcorr$ with $|\Dcorr|=\varepsilon N$, $\varepsilon\in[0,\frac12)$. The corrupted examples and their labels may depend arbitrarily on the data, but are fixed before the proxies are drawn.
\end{assumption}

This is the label-only version of strong contamination \citep{diakonikolas2019recent}. The uniform and high-norm flips of \Cref{sec:exp} satisfy it (threat model in \Cref{app:threat_model}). The indicator of $\Dcorr$, never observed by the learner, is the \emph{corruption mask}. Robustness here concerns selection -- a sampler is non-robust if it visits the corrupted examples more than their share, and a pruning rule is non-robust if its subset is more contaminated than a random one. 

For a score $s\ge0$ and smoothing $\xi\ge0$, the \emph{score sampler} is $q^s_i=(s_i+\xi\bar s)/((1+\xi)N\bar s)$ with $\bar s=\frac1N\sum_js_j$. Let $x_s=\bigl(\frac1{\varepsilon N}\sum_{i\in\Dcorr}s_i\bigr)/\bar s$ be the ratio of the corrupted examples' mean score to the dataset mean, and let $Q(q)=\sum_{i\in\Dcorr}q_i$ be the share of visits to corrupted examples ($\varepsilon$ under uniform sampling). With ties broken by a fixed random order, let $A_s$ be the fraction of (corrupted, clean) pairs in which the corrupted example scores higher, the score's corruption-detection AUROC, and let $S_\alpha$ be the $\alpha N$ highest-scoring examples, with contamination $\mathrm{cis}=|S_\alpha\cap\Dcorr|/(\alpha N)$.

\begin{definition}[Loss-aligned corruption]\label{def:loss_aligned}
Let $\ell^{h}$ be the per-example loss of a briefly trained proxy $h$, used as a score. The corruption is \emph{loss-aligned} if $A_{\ell^h}\ge\frac34$, the smallest AUROC at which \Cref{prop:suboptimality}(iii) makes pruning by the proxy loss no cleaner than random for some keep fraction (this can be estimated without labels---\Cref{sec:checks}). 
\end{definition}

\begin{theorem}[Visits and contamination are pinned by the score]\label{prop:suboptimality}
Under \Cref{assu:contamination} with $\varepsilon>0$:
\begin{enumerate}[label=(\roman*),nosep,leftmargin=*]
\item  \textbf{Visits.} For any $q>0$, $Q(q)/\varepsilon=\frac{1}{\varepsilon N}\sum_{i\in\Dcorr}\omega_i^{-1}$; for the score sampler, $Q(q^s)/\varepsilon=(x_s+\xi)/(1+\xi)$. Hence $Q(q^s)\ge\varepsilon$ for every $\xi$ if and only if $x_s\ge1$, i.e.\ the corrupted examples' mean score is at least the clean mean; this holds in particular if $s$ is non-decreasing in the loss and every corrupted loss exceeds every clean one.
\item \textbf{Two-sided AUROC bound.} With ${c=\alpha\,\mathrm{cis}}$,
$(\alpha-c)(\varepsilon-c)\le(1-A_s)\,\varepsilon(1-\varepsilon)$ and ${c\,(1-\varepsilon-\alpha+c)\le A_s\,\varepsilon(1-\varepsilon)},$
both attainable by a block ordering.
\item \textbf{Comparison with random selection.} If $A_s\ge1-\alpha(1-\alpha)$, then $\mathrm{cis}\ge\varepsilon$; if $A_s\le\alpha(1-\alpha)$, then $\mathrm{cis}\le\varepsilon$; if $A_s=1$, then $\mathrm{cis}=\min(1,\varepsilon/\alpha)$.
\end{enumerate}
\end{theorem}
\begin{proof}[Proof sketch]
(i) $q_i=\omega_i^{-1}/N$; sum over $\Dcorr$. (ii) Every clean example inside $S_\alpha$ outranks every corrupted one outside, and every corrupted example inside outranks every clean one outside; count the pairs. (iii) Both bounds are monotone in $c$ and reduce, at $\mathrm{cis}=\varepsilon$, to $A_s=1-\alpha(1-\alpha)$ and $A_s=\alpha(1-\alpha)$ (\Cref{app:prop_suboptimality}).
\end{proof}

Part (i) concerns samplers. An unbiased sampler changes only how often each example is visited, and for the score sampler the share of visits to the corrupted examples is a function of $x_s$ alone. Parts (ii) and (iii) concern pruning. Part (ii) bounds the contamination of the kept subset from both sides in terms of the AUROC; since both bounds are attained, they are the tightest bounds that depend on the AUROC only. Part (iii) compares the kept subset with a random one, whose expected contamination is $\varepsilon$: above an AUROC of $1-\alpha(1-\alpha)$ the subset is at least as contaminated as a random one, and below $\alpha(1-\alpha)$, at most as contaminated ($0.8125$ and $0.1875$ at $\alpha=0.25$). Because these thresholds do not depend on $\varepsilon$ and the AUROC does not depend on $\alpha$, a single measured AUROC per score determines the comparison for every keep fraction within a range. The theorem applies to the loss itself, to the mean loss rank across proxies, to multi-class EL2N under a mild margin condition, and to gradient-norm scores through their measured $x_s$ and AUROC (\Cref{app:prop_suboptimality}). It also identifies how a score avoids the failure: a score that is small on the examples every proxy ranks hardest has small $x_s$ and small AUROC. Selecting the smallest values of a loss score satisfies this condition, but retains only easy examples, the cleanest and least informative subsets (\Cref{sec:exp:static}). A useful score is therefore non-monotone in the loss: small on examples that all proxies find hardest, and large on the informative ones.


\section{Rank Disagreement}
\label{sec:method}

\subsection{Two rank statistics and their roles}
Let $h_1,\dots,h_K$ be independently trained proxies (random initialization and data order, optionally different architectures), each trained briefly on $\D$, and let $\rho_{ik}\in\{\frac1N,\dots,1\}$ be the rank of $\ell(h_k(x_i),y_i)$ among the $N$ losses, divided by $N$ ($1$: hardest). The \drisShort{} score and the total rank variance are:
\begin{equation}\label{eq:rankvar}
   \rankvar_{i}=\tfrac{1}{K}\textstyle\sum_{k}\bigl(\rho_{ik}-\overline{\rho}_{i}\bigr)^{2},\qquad \widehat b_i=\overline\rho_i(1-\overline\rho_i),
\end{equation}
with $\overline\rho_i=\frac1K\sum_k\rho_{ik}$. Let $Z_i=\mathbb 1[\ell(h_k(x_i),y_i)\ge\ell(h_k(x_J),y_J)]$ for a proxy index $k$ and an example $J$ drawn uniformly: $Z_i$ captures whether $i$ is at least as hard as a random example under a random proxy. Then $\Prob(Z_i=1\mid k)=\rho_{ik}$, so $\Var Z_i=\widehat b_i$, and the law of total variance splits it into a between-proxy part, $\rankvar_i$, and a within-proxy part, $W_i=\frac1K\sum_k\rho_{ik}(1-\rho_{ik})$. Ranks are essential: the variance of raw or z-scored losses selects subsets that are $53$--$55\%$ corrupted, against $7$--$9\%$ with ranks (\Cref{tab:score_ablation}), because a corrupted example's large losses vary between proxies while its rank stays at the top.

\begin{proposition}[Exact properties of rank statistics]\label{prop:ranks}
For any $K\ge2$ proxies whose ranks are permutations of $\{\frac1N,\dots,1\}$, with $v_N=\frac{N^2-1}{12N^2}$ and $\bar r$ the mean pairwise Spearman correlation of their loss rankings:
\begin{enumerate}[label=(\alph*),nosep,leftmargin=*]
\item \textbf{Total variance.} $\widehat b_i=\rankvar_i+W_i$, with $W_i\ge0$.
\item \textbf{Dataset means.} $\rankbar:=\frac1N\sum_i\rankvar_i={(1-\frac1K)(1-\bar r)v_N}$ and $\frac1N\sum_iW_i=2v_N\approx\frac16$.
\item \textbf{Pinning.} If all $K$ ranks of $i$ lie in $[1-\tau,1]$, then $\rankvar_i\le\min\{\tau^2/4,\ \tfrac{\tau}{1-\tau}\widehat b_i\}$ and $\widehat b_i\le\tau$.
\item \textbf{Single escape.} If $K-1$ ranks of $i$ equal $r$ and one equals $r-\delta$, then $\rankvar_i=\frac{K-1}{K^2}\delta^2$ and $\widehat b_i={\mu(1-\mu)}$ with $\mu=r-\frac\delta K$. Above a threshold $t$, $\rankvar$ admits $i$ iff $\delta>K\sqrt{t/(K-1)}$; $\widehat b$ admits it iff $|\mu-\frac12|<\sqrt{\frac14-t}$.
\item \textbf{Reflection.} $\rankvar$ is invariant under $\rho\mapsto1+\frac1N-\rho$, and $\widehat b$ up to $O(1/N)$.
\end{enumerate}
\end{proposition}
The proof is algebra on permutations (\Cref{app:prop_ranks}). By (c), an example that every proxy ranks among the hardest scores low under both statistics, so neither is loss-monotone and both escape \Cref{prop:suboptimality}. Parts (a)--(e) predict which statistic suits which use, and where both fail.

For \emph{sampling}, what matters is the ratio $x$ of the corrupted examples' mean score to the dataset mean. By (a) and (b), the within-proxy part contributes the constant $2v_N \approx \tfrac16$ to the dataset mean of TRV, against $\rankbar=(1-\frac1K)(1-\bar r)v_N$ for \drisShort{}, which shrinks as the proxies agree ($0.014$ against $\bar b = 0.18$ in the setting of \Cref{tab:auroc}). The same split holds on the corrupted examples: there $\rankvar$ averages $0.006$, but $W$ averages $0.12$, about $0.7$ of $\bar W$, because their mean ranks sit near $0.85$, not at $1$. Hence $x_{\mathrm{TRV}}=(\rankvar_{\mathrm{corr}}+W_{\mathrm{corr}})/(\rankbar+\bar W)$ stays large ($0.69$), whereas $x_{\drisShort{}}=\rankvar_{\mathrm{corr}}/\rankbar$ is set mostly by the escaping tail ($0.44$): \drisShort{} has the smaller $x$ and, by \Cref{prop:suboptimality}(i), is the better sampler.

For \emph{pruning}, what matters is the order of the scores. By (d), a corrupted example that escapes the top ranks under only one proxy, the most common escape, has a large between-proxy variance and enters \drisShort{}'s top set, but its mean rank stays far from $\frac12$, and so it stays out of TRV's. TRV therefore has the lower AUROC and, by \Cref{prop:suboptimality}(ii)--(iii), the cleaner subsets. By (e), neither statistic distinguishes the consistently hardest from the consistently easiest examples, so both lose their edge once the keep fraction reaches deep into the low-score pool.

\subsection{Separation for large ensembles}
Let $\mathcal P$ be the distribution of trained proxies induced by the randomness of the learner's pipeline (initialization, data order and, optionally, architecture), $h_1,\dots,h_K\stackrel{\text{i.i.d.}}{\sim}\mathcal P$ and write $\rho_i(h)$ for the normalized loss rank of example $i$ under proxy $h$.
\begin{assumption}[Pinned corrupted ranks]\label{assu:adv_concentration}
A bulk $\Dbulk\subseteq\Dcorr$ of size $(1-\atrim)\varepsilon N$ satisfies $\Prob_{h\sim\mathcal P}(\rho_i(h)\ge1-\tau)\ge1-\gamma$ for all $i\in\Dbulk$; $\Dtail=\Dcorr\setminus\Dbulk$ is the escaping tail.
\end{assumption}
\begin{assumption}[Boundary disagreement]\label{assu:bdry_disagreement}
Every $j$ in a set $\Dbdry\subseteq\Dclean$ has $\Var_{h\sim\mathcal P}[\rho_j(h)]\ge\tau^2_{\mathrm{bdry}}$.
\end{assumption}

\Cref{assu:adv_concentration} says that all but a fraction $\atrim$ of the corrupted examples are among the hardest $\tau$ fraction under all but a fraction $\gamma$ of proxies; the remaining corrupted examples form the escaping tail. It holds when briefly trained proxies fit shared class structure before they memorize contradicting labels \citep{arpit2017closer,liu2020early}, so a wrong label keeps a high loss under every proxy. \Cref{assu:bdry_disagreement} formalizes that small changes in initialization or data order move a proxy's decision surface across near-boundary clean examples. The first fails when the proxies memorize the noise and the second when the boundary is sparse; both failures are visible in the disagreement histogram (\Cref{sec:checks}).

\begin{theorem}[High-quantile separation]\label{thm:main}
Let $\rK=\sqrt{\log(2N/\delta)/(2K)}$. Under \Cref{assu:contamination,assu:adv_concentration,assu:bdry_disagreement}, with probability at least $1-\delta$ over the draw of the $K$ proxies, $\rankvar_i\le(1-\frac1K)(\frac{\tau^2}{4}+\gamma)+\rK$ for all $i\in\Dbulk$ and $\rankvar_j\ge(1-\frac1K)\tau^2_{\mathrm{bdry}}-\rK$ for all $j\in\Dbdry$. If $\Delta':=(1-\frac1K)(\tau^2_{\mathrm{bdry}}-\frac{\tau^2}{4}-\gamma)>2\rK$ and $|\Dbdry|\ge\alpha N$, then $S_\alpha$ contains no bulk-corrupted example, and if the escaping tail has average rank variance at most $v_{\mathrm{tail}}$,
\begin{equation}\label{eq:contam}
\mathrm{cis}\;\le\;\frac{\atrim\,\varepsilon}{\alpha}\,\min\Bigl(1,\frac{v_{\mathrm{tail}}}{\tau^2/4+\gamma}\Bigr).
\end{equation}
TRV satisfies the same statements with bulk ceiling $\tau+\gamma$ and a mid-difficulty condition on the clean side that \Cref{assu:bdry_disagreement} implies (\Cref{prop:band}).
\end{theorem}

The proof combines McDiarmid's inequality with $\E\rankvar_i=(1-\frac1K)\Var_{\mathcal P}[\rho_i]$ and a Markov count of the escaping tail (\Cref{app:guarantees}). The theorem makes the mechanism precise: the corrupted bulk concentrates near zero disagreement, the clean boundary does not, and the separation sharpens as $K$ grows. The radius $\rK$ is a worst-case union bound, vacuous at $K=3$; certifying separation at CIFAR scale requires at least several hundred proxies even under the most favourable constants. For the deployed $K=3$ ensemble, the operative statement is therefore \Cref{prop:suboptimality} with measured quantities: TRV's AUROC certifies subsets cleaner than random over a range of keep fractions, and \drisShort{}'s $x$ fixes its visit share exactly (\Cref{sec:exp:checks}). Both can be estimated before training (\Cref{sec:checks}), and \Cref{thm:main} explains why they are small; its concentration is already visible at small $K$ (\Cref{fig:K_ablation}).

\subsection{Online \drisShort{} and TRV}
\drisShort{}-online is the score sampler of $\rankvar$,
\begin{equation}\label{eq:dris_online}
q_i=\frac{\rankvar_i+\xi\rankbar}{(1+\xi)N\rankbar},\qquad \omega_i=(Nq_i)^{-1},
\end{equation}
with $\xi=0.1$ fixed before any run. Let $g=\omega_i\nabla f_i(w)$, $i\sim q$, and $g_U=\nabla f_J$, $J$ uniform.
\begin{proposition}[Online \drisShort{}]\label{prop:online}
For every ${\xi>0}$ and $w$: (i) $\E_q[g]=\nabla\widehat R(w)$; (ii) $\omega_i\le\kappa:=\frac{(1+\xi)\rankbar}{\min_j\rankvar_j+\xi\rankbar}\le1+\frac1\xi$; (iii) $\E_q\|g\|^2=\kappa_{\mathrm{eff}}(w)\,\E\|g_U\|^2$ with $\kappa_{\mathrm{eff}}=\frac{\sum_i\omega_i\|\nabla f_i\|^2}{\sum_i\|\nabla f_i\|^2}\le\kappa$, so an SGD step on an $L$-smooth $\widehat R$ decreases it in expectation if ${\eta<2\|\nabla\widehat R\|^2/(L\kappa_{\mathrm{eff}}\E\|g_U\|^2)}$.
\end{proposition}

The proof follows from the definitions (\Cref{app:prop_online}). The reweighting is thus safe inside SGD: the step is unbiased, no weight exceeds $1+1/\xi$, and the usual SGD guarantees hold with a constant at most $\kappa$ larger than under uniform sampling. The bound has a price: visiting the corrupted examples less often forces their average weight up, to at least $\varepsilon/Q$ by the arithmetic--harmonic mean inequality, and these large weights limit the usable step size, a trade-off set by $\xi$ (\Cref{app:online_ablations,app:stability}). The scores are computed once, before the target is trained, and are not refreshed.

TRV-online, a control that uses $\widehat b$ in the same sampler, differs from \drisShort{}-online only by the within-proxy part $W$ and, by \Cref{prop:ranks}(b), should suppress the corrupted examples less.

\subsection{Pre-deployment checks}
\label{sec:checks}
The theory's inputs can be checked before the target is trained. \emph{(C1) Disagreement histogram.} \Cref{assu:adv_concentration,assu:bdry_disagreement} predict many examples near $\rankvar=0$ and a spread-out clean population (\Cref{fig:K_ablation}); a histogram collapsed at zero implies the proxies are too similar or have memorized the labels. \emph{(C2) Label-free $\varepsilon$, AUROC and $x$.} The corruption mask is unknown, but a pseudo-mask approximates it: the examples whose given label gets probability below $\frac12$ from every proxy. Its size estimates $\varepsilon$, and using it in place of the mask in $x_s$ and $A_s$ estimates both quantities, and with them the verdicts of \Cref{prop:suboptimality}, for any candidate score (\Cref{tab:auroc}). The pseudo-mask needs proxies scored before they fit the corrupted labels, as ours are at epoch $20$. A third check (C3), on the scoring epoch, is in \Cref{app:diagnostics}.

\section{Experiments}
\label{sec:exp}
\label{sec:experiments}

\paragraph{Setup.} \emph{Data and targets:} CIFAR-10 with VGG-19-BN (160 epochs) and CIFAR-100 with ResNet-18 (200 epochs) \citep{krizhevsky2009learning,simonyan2015very,he2016deep}; Food-101 \citep{bossard2014food} with ResNet-50 as a natural-noise check. \emph{Noise:} \emph{uniform} flips of a random $\varepsilon$-fraction of labels to another class; \emph{high-norm} flips of the top-$\varepsilon$ examples by gradient norm under a separately trained ResNet-20 noise model; masks are shared by all methods. \emph{Proxies:} $K=3$ \{ResNet-20, MobileNetV2$_{0.5\times}$, ShuffleNetV2$_{0.5\times}$\}, trained on the corrupted data and scored at epoch 20 \citep{paul2021deep}. \emph{Baselines:} statically, Random, EL2N, Consensus-loss (the mean rank $\overline\rho$, the mean-aggregated counterpart of \drisShort{}), Forgetting, and AUM, all from the same proxy runs; online, uniform SGD, gradient-norm IS, and the RHO-LOSS score (\Cref{app:rho_loss_breakdown}), all inside the same sampler. Each static rule keeps its top $\alpha N$ examples, with $\alpha=0.25$, and trains for $1/\alpha$ times as many epochs, so the number of optimizer steps matches full-data training (step parity); the proxies add $2.1\%$ of the target's FLOPs ($4.2\%$ with the 40-epoch schedule we ran; \Cref{app:cost}). Seed counts are given in the table captions; further details are in \Cref{app:exp_details}.

\subsection{Verification of the theory's predictions}
\label{sec:exp:checks}
\begin{table}[h]
\centering
\caption{\Cref{prop:suboptimality} on measured scores (CIFAR-10 high-norm-$25\%$, deployed $K=3$ proxies, $\alpha=0.25$; \Cref{app:table1}). \emph{True}/\emph{est.}: AUROC and $x$ with the corruption mask/label-free pseudo-mask. \emph{Allowed}: cis range permitted by (ii). \emph{(iii)}: verdict against random selection and the range of $\alpha$ where it holds.}
\label{tab:auroc}
\scriptsize
\setlength{\tabcolsep}{2pt}
\begin{tabular}{l cc cc c c c}
\toprule
& \multicolumn{2}{c}{AUROC} & \multicolumn{2}{c}{$x$} & & & \\
\cmidrule(lr){2-3}\cmidrule(lr){4-5}
Score & true & est. & true & est. & allowed & cis & (iii), $\alpha$-range \\
\midrule
loss, 1 proxy       & $.929$ & $.977$ & $2.86$ & $3.23$ & $[.54,.95]$ & $.766$ & $\ge\varepsilon$, $.08$--$.92$ \\
loss, mean          & $.976$ & $.991$ & $2.81$ & $3.00$ & $[.73,.98]$ & $.869$ & $\ge\varepsilon$, $.03$--$.97$ \\
mean rank $\overline\rho$ & $.977$ & $.995$ & $1.68$ & $1.73$ & $[.74,.98]$ & $.869$ & $\ge\varepsilon$, $.03$--$.97$ \\
\drisShort{} $\rankvar$ & $.254$ & $.167$ & $.44$ & $.22$ & $[0,.33]$ & $.073$ & undecided \\
TRV $\widehat b$      & $.180$ & $.122$ & $.69$ & $.62$ & $[0,.24]$ & $.036$ & $\le\varepsilon$, $.24$--$.76$
 \\
\bottomrule
\end{tabular} 
\end{table}

\Cref{tab:auroc} measures both quantities of \Cref{prop:suboptimality} on the proxy losses of the main CIFAR-10 setting, with the true mask and with the pseudo-mask. Every loss-based score, a single proxy's loss, the mean loss, and the mean rank, has AUROC above $0.92$ and $x>1.6$: (iii) places its top-$\alpha$ subsets at or above random contamination for every $\alpha$ in $[0.08,0.92]$, and (i) says a sampler built on it visits corrupted examples $1.6$--$2.7$ times as often as uniform sampling. Both rank statistics sit on the other side, in the order \Cref{prop:ranks} predicts: \drisShort{} has the lower $x$ ($0.44$ vs.\ $0.69$), TRV the lower AUROC ($0.18$ vs.\ $0.25$), below the threshold $0.1875$ that guarantees a subset cleaner than random. Each realized contamination lies in the interval (ii) allows. The pseudo-mask estimates $\varepsilon$ as $0.23$ (true $0.25$) with precision $0.88$ and recall $0.80$, and its estimates reproduce every ordering and every verdict; they are optimistic for the rank statistics, which are near zero on exactly the examples it flags, so they support the orderings rather than the absolute values.

\Cref{fig:K_ablation} shows the structure \Cref{assu:adv_concentration,assu:bdry_disagreement} posit: for every $K$ the corrupted examples stay concentrated near zero disagreement, while from $K=4$ on the clean examples form a mode away from zero. As \Cref{thm:main} predicts, the empirical gap between the two populations grows with $K$ as $\rankvar$ concentrates around its expectation; it saturates by $K\approx8$, far ahead of the worst-case radius. \Cref{fig:auroc_K} shows what increasing $K$ does to the AUROC of \drisShort{} and TRV and, for contrast, of the mean rank, the loss-monotone consensus of the same proxies. As the gap grows, \drisShort{}'s AUROC falls from $0.28$ at $K=2$ to $0.13$ at $K=32$ and crosses below the threshold $0.1875$ of \Cref{prop:suboptimality}(iii) at $K=16$; TRV is below it from $K=2$ on. Averaging more proxies has the opposite effect: the mean rank moves further to the wrong side, from $0.90$ to $0.98$.

\begin{figure*}[t]
\centering
\includegraphics[width=0.95\textwidth]{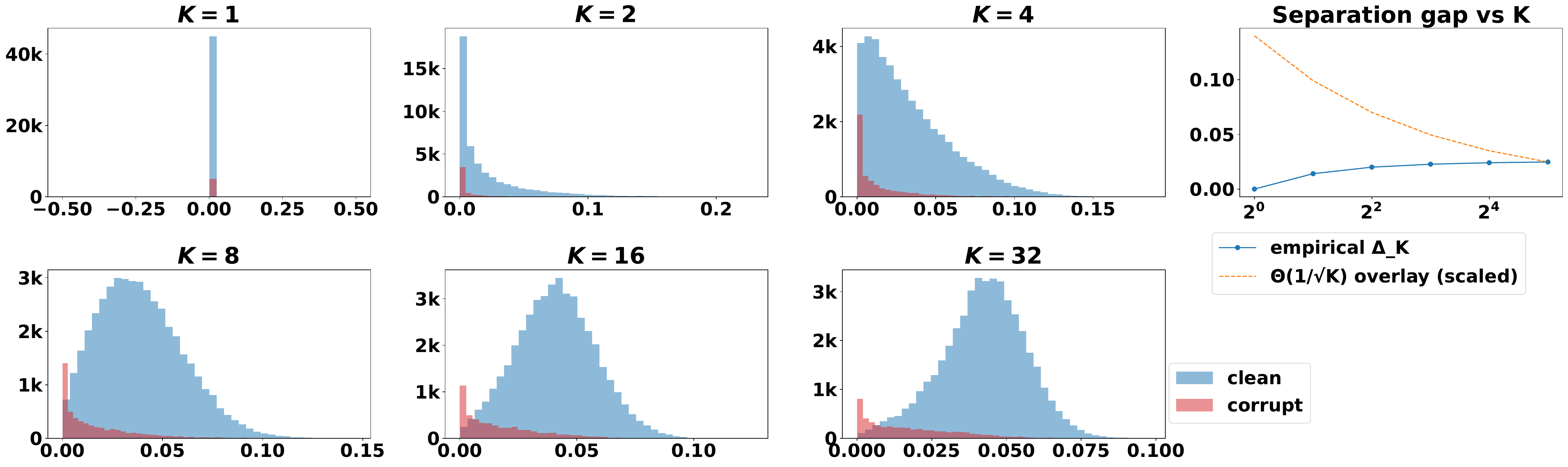}
\caption{Check (C1) and the mechanism of \Cref{thm:main} (CIFAR-10, high-norm-$10\%$, a separately trained 32-proxy ensemble). \textbf{Left:} $\rankvar$ histograms on $\Dclean$ (blue) and $\Dcorr$ (red) for $K\in\{1,2,4,8,16,32\}$: the corrupted examples stay at near-zero disagreement while the clean examples move away from zero as $K$ grows. \textbf{Right:} mean gap $\widehat\Delta_K=\overline{\rankvar}|_{\Dclean}-\overline{\rankvar}|_{\Dcorr}$ (solid) with a $\Theta(K^{-1/2})$ reference curve scaled for display (dashed).}
\label{fig:K_ablation}
\end{figure*}

\subsection{Online importance sampling}
\label{sec:exp:online}
\begin{table*}[t]
\centering
\caption{Online importance sampling on all $N$ examples (CIFAR-10/VGG-19-BN, $\xi=0.1$): final test accuracy (\%, mean$\pm$std over 3 seeds). \emph{ES}: uniform SGD at its best epoch \emph{chosen on the test set}, an upper bound for early stopping as reference. \emph{GN-IS} (TRV / \drisShort{}): gradient-norm IS. $Q/\varepsilon$: share of sampled visits on corrupted examples relative to uniform sampling. \emph{Fit} (Unif.\ / TRV / \drisShort{}): fraction of corrupted examples that the final model predicts as their (wrong) given label. U: uniform noise; HN: high-norm noise.}
\label{tab:online_is}
\small
\setlength{\tabcolsep}{4pt}
\begin{tabular}{l ccc cc c c}
\toprule
Setting & Uniform & ES & GN-IS & TRV-online & \drisShort{}-online & $Q/\varepsilon$ (T / DR) & Fit (U / T / DR) \\
\midrule
Clean   & $\mathbf{93.70\pm0.21}$ & $93.80$ & $90.81\pm0.11$ & $93.64\pm0.20$ & $92.47\pm0.27$ & --- & --- \\
U 10\%  & $87.78\pm0.55$ & $88.80$ & $89.30\pm0.26$ & $89.16\pm0.30$ & $\mathbf{90.36\pm0.53}$ & $.46$ / $.31$ & $.60$ / $.32$ / $.08$ \\
U 25\%  & $77.69\pm1.10$ & $84.80$ & $85.98\pm1.50$ & $79.89\pm1.30$ & $\mathbf{87.00\pm0.40}$ & $.70$ / $.47$ & $.56$ / $.43$ / $.09$ \\
HN 10\% & $86.59\pm0.22$ & $86.88$ & $10.00\pm0.00$ & $87.50\pm0.11$ & $\mathbf{89.28\pm0.13}$ & $.51$ / $.36$ & $.97$ / $.55$ / $.17$ \\
HN 25\% & $74.91\pm0.50$ & $81.61$ & $10.00\pm0.00$ & $76.02\pm0.64$ & $\mathbf{84.52\pm0.26}$ & $.72$ / $.49$ & $.94$ / $.62$ / $.15$
 \\
\bottomrule
\end{tabular}
\end{table*}

\paragraph{Results.} \drisShort{}-online improves on uniform SGD in every corrupted setting of \Cref{tab:online_is}, by over $2.5$\,pp at $10\%$ noise and by over $9.3$\,pp at $25\%$. On the test-set-selected early-stopping bound the gain is $1.6$--$2.9$\,pp and the clean-data cost is $1.2$\,pp. Re-drawing the proxies does not change the result: $84.35\pm0.19\%$ at high-norm-$25\%$ (\Cref{app:proxy_seeds}). TRV-online gains only $0.9$--$2.2$\,pp, consistent with \Cref{prop:ranks} for the statistic with the larger $x$. The visit share is what carries the effect: \drisShort{}-online at $\xi = 1$ has the same visit share as TRV-online at $\xi = 0.1$ ($Q/\varepsilon = 0.722$ vs.\ $0.719$, \Cref{tab:visits_fit}) and reaches the same accuracy ($76.1$ vs.\ $76.0\%$) and nearly the same fit ($0.58$ vs.\ $0.62$); two samplers with equal visit shares behave alike, as \Cref{prop:suboptimality}(i) suggests, and \drisShort{} wins because at the default $\xi$ its share is lower. Gradient-norm IS realizes the predicted amplification, sending $3.3$ times the uniform share of visits to the corrupted examples at $25\%$, and ends at chance under both high-norm settings. Under uniform flips it over-visits them too ($5.9\times$, $2.5\times$), yet reaches $89.3\%$ and $86.0\%$ with fit below $5\%$. The gain is robust to how the corruption is chosen: with a ResNet-18 or a VGG-16-BN noise model, whose top-$10\%$ sets overlap the default one by only $26$--$28\%$, \drisShort{}-online reaches $83.4\%$ and $84.1\%$ against $73.5\%$ and $74.4\%$ for uniform SGD, and under volatile-margin flips of mostly low-norm examples it also reaches $87.4\%$ against $76.2\%$; static TRV stays the best pruning rule under all three (\Cref{app:noise_models}).

\paragraph{Mechanism: visits, then memorization.} The first link is exact: \Cref{prop:suboptimality}(i) with the measured $x$ predicts the realized visit share of TRV-online and of \drisShort{} at four smoothings to within $0.005$ (\Cref{tab:visits_fit}). The second is empirical. Uniform SGD ends up fitting $56$--$97\%$ of the wrong labels, \drisShort{}-online $8$--$17\%$, i.e.\ $13$--$18\%$ of uniform SGD's fit in every setting, and TRV-online $32$--$62\%$. Wrong labels are memorized late, as the learning rate decays (\Cref{fig:memorization}), consistent with networks fitting shared structure before individual labels \citep{arpit2017closer,li2020gradient}. Since the expected update is unchanged, the gain is a finite-horizon effect: a pinned corrupted example arrives about $11\times$ less often with an $11\times$ larger weight, and these sparse exposures are absorbed by the clean signal before the memorization phase (\Cref{app:online_ablations}).

\paragraph{Stability.} The weight bound $1+1/\xi=11$ of \Cref{prop:online} is the cost of suppression: early in training the step's second moment is about $2.8\times$ that of uniform SGD. In the five settings of \Cref{tab:online_is} every seed converged, but at learning rate $0.1$ without warm-up one of three seeds collapsed in the first epoch with the VGG-16-BN noise model and at high-norm-$40\%$; uniform SGD and TRV-online never did. A $5$-epoch warm-up for both methods removes this at no cost: all runs converge, and \drisShort{}-online gains $9.7$, $9.6$ and $14.9$\,pp over uniform SGD in the default, VGG-16-BN and high-norm-$40\%$ settings (\Cref{app:stability}).

\subsection{Static pruning}
\label{sec:exp:static}
\label{sec:exp:cifar_pruning}
\label{sec:exp:regime}
\begin{table*}[t]
\centering
\caption{Static pruning at $\alpha=0.25$: test accuracy (\%, mean$\pm$std) and, at high-norm-$25\%$, the corrupted fraction of the selected subset (\emph{cis}; random $=0.25$). Consensus-loss is the mean rank $\overline\rho_i$, TRV is the total rank variance $\overline\rho_i(1-\overline\rho_i)$. Seeds: $n=7$ for cells in \Cref{app:paired_t_headline}, $n=3$ elsewhere.}
\label{tab:cifar_pruning}
\footnotesize
\setlength{\tabcolsep}{3pt}
\resizebox{\linewidth}{!}{\begin{tabular}{l ccccc c ccccc}
\toprule
& \multicolumn{5}{c}{\textsc{CIFAR-10 / VGG-19-BN}} && \multicolumn{5}{c}{\textsc{CIFAR-100 / ResNet-18}} \\
\cmidrule(lr){2-6}\cmidrule(lr){8-12}
Method & Clean & Unif.\ 10\% & HN\ 10\% & HN\ 25\% & cis && Clean & Unif.\ 10\% & HN\ 10\% & HN\ 25\% & cis \\
\midrule
Random               & $\mathbf{87.78\pm0.26}$ & $81.84\pm0.55$ & $82.81\pm0.15$ & $74.66\pm0.54$ & $0.252$ && $63.09\pm0.33$ & $54.81\pm0.40$ & $58.76\pm0.68$ & $51.53\pm0.59$ & $0.252$ \\
\multicolumn{12}{l}{\emph{loss-monotone}} \\
EL2N (ens.)          & $83.58\pm0.36$ & $27.53\pm0.70$ & $46.26\pm4.46$ & $\phantom{0}9.39\pm0.61$ & $0.866$ && $30.60\pm0.59$ & $\phantom{0}7.43\pm0.13$ & $12.25\pm0.58$ & $\phantom{0}1.54\pm0.19$ & $0.862$ \\
Consensus-loss ($\overline\rho$) & $78.79\pm0.49$ & $28.20\pm0.95$ & $41.87\pm5.23$ & $\phantom{0}9.00\pm1.02$ & $0.871$ && $28.68\pm1.19$ & $\phantom{0}6.19\pm0.42$ & $10.59\pm0.64$ & $\phantom{0}1.11\pm0.04$ & $0.895$ \\
\multicolumn{12}{l}{\emph{training dynamics}} \\
Forgetting           & \underline{$87.39\pm0.58$} & $82.20\pm0.53$ & $82.54\pm1.07$ & $72.35\pm0.90$ & $0.300$ && $58.33\pm0.38$ & $60.36\pm0.30$ & $59.84\pm0.16$ & $\mathbf{61.88\pm0.62}$ & $0.037$ \\
AUM (top-$\alpha$)   & $80.60\pm0.31$ & $82.86\pm0.33$ & $79.73\pm1.58$ & $79.69\pm0.28$ & $\mathbf{0.001}$ && $55.32\pm0.12$ & $55.14\pm0.34$ & $54.50\pm0.11$ & $53.32\pm0.14$ & $\mathbf{0.000}$ \\
\multicolumn{12}{l}{\emph{rank statistics (non-monotone in the loss)}} \\
\drisShort{} $\rankvar$ & $87.34\pm0.13$ & \underline{$85.98\pm0.35$} & \underline{$84.55\pm0.94$} & \underline{$81.86\pm0.82$} & $0.072$ && \underline{$63.65\pm0.24$} & \underline{$62.02\pm0.62$} & \underline{$63.15\pm0.94$} & $60.96\pm0.68$ & $0.055$ \\
TRV $\overline\rho(1-\overline\rho)$ & $86.42\pm0.55$ & $\mathbf{87.13\pm0.30}$ & $\mathbf{86.62\pm0.12}$ & $\mathbf{85.93\pm0.22}$ & $0.036$ && $\mathbf{65.45\pm0.48}$ & $\mathbf{63.78\pm0.25}$ & $\mathbf{65.01\pm0.63}$ & \underline{$61.55\pm0.41$} & $0.025$
 \\
\bottomrule
\end{tabular}}
\end{table*}

\paragraph{Magnitude collapses; rank disagreement does not.} EL2N and Consensus-loss select $86$--$90\%$-corrupted subsets at high-norm-$25\%$ and fall to $1$--$9\%$ accuracy (\Cref{tab:cifar_pruning}), and they collapse under uniform noise as well ($6$--$28\%$ at $10\%$), as \Cref{prop:suboptimality}(iii) predicts from their AUROC. TRV, the same mean ranks folded by $\overline\rho(1-\overline\rho)$, selects a $4\%$-corrupted subset: only the fold differs. Both rank statistics beat Random in every corrupted cell, TRV by $3.8$--$11.3$\,pp and \drisShort{} by $1.7$--$9.4$\,pp; the \drisShort{} gain at CIFAR-10 high-norm-$25\%$ is $7.21$\,pp in a paired analysis over $7$ seeds ($t=15.5$; \Cref{app:paired_t_headline}). With a fresh proxy ensemble per run seed, TRV reaches $86.04\pm0.11\%$ and the \drisShort{} score $81.50\pm0.81\%$ there (\Cref{app:proxy_seeds}).

\paragraph{The roles predicted by the theory.} TRV, the statistic with the lower AUROC, is the more accurate threshold in every corrupted cell, by $0.6$--$4.1$\,pp, and has the cleaner subsets ($0.036$ vs.\ $0.072$ at CIFAR-10 high-norm-$25\%$, $0.025$ vs.\ $0.055$ on CIFAR-100). The between-proxy information matters for TRV as well: computed from a single proxy, for which it reduces to a function of one proxy's ranks, TRV reaches $83.1\%$ at cis $0.085$ instead of $85.9\%$ at $0.036$ (\Cref{app:exp:band}). About two thirds of the contamination gap between the two statistics are single-proxy escapes, which \drisShort{} admits and TRV does not (\Cref{prop:ranks}(d)). Online, the order reverses as $x$ predicts (\Cref{sec:exp:online}), so one proxy ensemble supplies both rules, \drisShort{} for sampling and TRV for pruning.

\paragraph{Clean but uninformative is not enough.} AUM selects nearly clean subsets (cis $\le0.001$) and the small-loss posterior of DivideMix (GMM-select) the cleanest of all ($0.0005$), but both keep mostly easy examples and reach only $79.7\%$ and $77.4\%$ at CIFAR-10 high-norm-$25\%$, against $85.9\%$ for TRV (\Cref{app:dividemix,tab:cifar_pruning}). Rank disagreement keeps the informative boundary while excluding the corrupted bulk. Full DivideMix \citep{li2020dividemix}, which changes the training objective rather than the selection and trains on all $N$ examples, reaches $84.8\%$, close to online \drisShort{}'s $84.5\%$, at $4.8\times$ its wall-clock (\Cref{app:dividemix}). Forgetting is competitive under noise on CIFAR-100 but costs $4.8$\,pp on clean data.

\paragraph{Scope of the static rules.} Without corruption, \drisShort{} is within $0.44$\,pp of Random on CIFAR-10 and above it on CIFAR-100; TRV costs $1.4$\,pp on clean CIFAR-10. Over keep fractions and corruption rates (\Cref{tab:alpha_eps}), TRV is the best rule in $6$ of $9$ cells, including every cell at $\alpha=0.25$, while AUM prevails at large keep fractions and high corruption, where, by \Cref{prop:ranks}(e), rank statistics cannot separate the consistently hardest from the consistently easiest examples. Under $40\%$ human noise (CIFAR-10N \emph{worst}) the order of the two rank statistics reverses: \drisShort{} prunes better ($71.4$ vs.\ $66.8\%$), as its lower measured AUROC predicts ($0.38$ vs.\ $0.41$), and AUM is the best rule ($76.9\%$); the measured AUROC, rather than the single-escape heuristic of \Cref{prop:ranks}(d), identifies the better threshold (\Cref{app:cifar10n}). The \drisShort{} score is at or above Random for every $\varepsilon\le0.45$ and fails at the breakdown point $\varepsilon=\frac12$, where TRV stays $7.5$\,pp above Random (\Cref{fig:breakdown_curve}). On Food-101, where fine-grained labels are ambiguous rather than contradicting, \drisShort{} matches Random ($82.88$ vs.\ $82.67\%$) and TRV is $1.3$\,pp below it, while AUM and EL2N lose $4.9$ and $8.4$\,pp (\Cref{app:food101}).


\section{Discussion}
\label{sec:limitations}
\label{sec:exp:scope}

\paragraph{Limitations.} Our claims are conditional on loss-aligned corruption, which must be measured per dataset with checks (C1)--(C2): CIFAR-10N has it and Food-101 does not (\Cref{sec:exp:static}). Apart from CIFAR-10N, our noise models are synthetic flips, and transfer to ImageNet-scale data and proxies is untested (\Cref{app:cost}). The method needs proxies that disagree on the clean boundary but have not memorized the wrong labels; we fix their architectures, training length and scoring epoch in advance and give no rule for a new dataset. Proxies scored too late make the selection corruption-enriched (\Cref{tab:scoring_epoch}), which check (C3) detects without labels but mainly for larger ensembles; proxies that are too similar collapse the histogram of check (C1), though homogeneous small models (\Cref{tab:proxy_diversity}) and linear probes on pretrained features (\Cref{app:food101}) both suffice. \Cref{thm:main} explains the mechanism but does not certify $K=3$, whose low contamination is an empirical finding that \Cref{prop:suboptimality}(iii) confirms only for TRV, and the memorization reduction of the unbiased sampler is a finite-horizon hypothesis supported by fit curves. The regime where AUM is better and the breakdown of \drisShort{} at $\varepsilon=\frac12$ are described in \Cref{sec:exp:static}. Online \drisShort{} costs $1.2$\,pp on clean data, needs a learning-rate warm-up at $0.1$, and is evaluated on CIFAR-10 only. Active Bias, confident learning, CRUST and Prune4ReL were not run.

\paragraph{Conclusion.} Under loss-aligned label contamination, selection scores that grow with the loss provably feed the wrong labels to the learner: as samplers, whenever the corrupted examples' mean score exceeds the dataset mean, and as thresholds, whenever their AUROC exceeds $1-\alpha(1-\alpha)$; both hold for every loss-based score we measure. The variance of loss ranks across a few cheap proxies avoids this, and one decomposition assigns the roles: the between-proxy part, \drisShort{}, samples and the total rank variance prunes. The same proxies yield label-free estimates of $\varepsilon$, the AUROC and $x$, so a practitioner can check before training whether any score will beat random selection. Since the scores change only which examples SGD sees, they compose with pretrained representations and with loss- or label-correcting methods; combining them with robust learners, choosing the scoring epoch automatically, and scaling the proxies to ImageNet-size data are the natural next steps.

\begin{ack}
Computations and data handling were enabled by the Berzelius resource provided by the Knut and Alice Wallenberg Foundation at the National Supercomputer Centre through project Berzelius-2026-100, and by the Alvis resource provided by Chalmers e-Commons at Chalmers through project NAISS 2026/3-305 and NAISS 2026/4-73. Csongor Horváth and Prashant Singh acknowledge support from the Swedish Research Council through grant agreement no.~2023-05593. Prashant Singh also acknowledges funding from the Knut and Alice Wallenberg Foundation through the Wallenberg AI, Autonomous Systems and Software Program (WASP) and the Data-Driven Life Science (DDLS) NEST project ``TIMED''. 

The authors thank Aleksandr Karakulev for his valuable feedback on early drafts of the paper.
\end{ack}

\bibliographystyle{plainnat}
\bibliography{references}

\appendix
\section{Proofs of the Main-Text Results}
\label{app:proofs}

Throughout, ties in a score are broken by the fixed random order of \Cref{sec:setting}, so every score induces a strict order on $\D$; ``$i$ outranks $j$'' refers to this order, and both $A_s$ and $S_\alpha$ are read from it. We treat $\varepsilon N$ and $\alpha N$ as integers.

\subsection{Proof of \Cref{prop:suboptimality}}
\label{app:prop_suboptimality}

\begin{proof}
\emph{(i) Visits.} Since $\omega_i=(Nq_i)^{-1}$, each visit probability is $q_i=\omega_i^{-1}/N$. Summing over the corrupted examples,
\[
\frac{Q(q)}{\varepsilon}=\frac1\varepsilon\sum_{i\in\Dcorr}q_i=\frac1{\varepsilon N}\sum_{i\in\Dcorr}\omega_i^{-1}.
\]
For the score sampler, the corrupted scores sum to $\varepsilon N\,x_s\bar s$ by the definition of $x_s$, so
\[
Q(q^s)=\sum_{i\in\Dcorr}\frac{s_i+\xi\bar s}{(1+\xi)N\bar s}=\frac{\varepsilon N\,x_s\bar s+\varepsilon N\,\xi\bar s}{(1+\xi)N\bar s}=\varepsilon\,\frac{x_s+\xi}{1+\xi}.
\]
For every $\xi\ge0$ the right-hand side is at least $\varepsilon$ if and only if $x_s\ge1$. To see that $x_s\ge1$ means that the corrupted mean score $\bar s_{\mathrm{corr}}$ is at least the clean mean $\bar s_{\mathrm{clean}}$, write the dataset mean as the mixture $\bar s=\varepsilon\bar s_{\mathrm{corr}}+(1-\varepsilon)\bar s_{\mathrm{clean}}$; then $\bar s_{\mathrm{corr}}\ge\bar s$ is equivalent to $(1-\varepsilon)\bar s_{\mathrm{corr}}\ge(1-\varepsilon)\bar s_{\mathrm{clean}}$. Finally, if $s$ is non-decreasing in the loss and every corrupted loss exceeds every clean one, every corrupted score is at least every clean score, so in particular $\bar s_{\mathrm{corr}}\ge\bar s_{\mathrm{clean}}$.

\emph{(ii) Two-sided AUROC bound.} Let $c=\alpha\,\mathrm{cis}$, so that $S_\alpha$ contains $cN$ corrupted examples. The dataset splits into four groups:
\[
\underbrace{cN}_{\text{corrupted, inside}},\qquad
\underbrace{(\alpha-c)N}_{\text{clean, inside}},\qquad
\underbrace{(\varepsilon-c)N}_{\text{corrupted, outside}},\qquad
\underbrace{(1-\varepsilon-\alpha+c)N}_{\text{clean, outside}}.
\]
Of the $\varepsilon(1-\varepsilon)N^2$ (corrupted, clean) pairs, a fraction $A_s$ is correctly ordered (the corrupted example outranks the clean one) and a fraction $1-A_s$ is mis-ordered. Since $S_\alpha$ consists of the top $\alpha N$ examples, every example inside it outranks every example outside it, so two kinds of pairs are fixed by the subset alone. Every pair of a clean example inside and a corrupted example outside is mis-ordered, which gives $(\alpha-c)(\varepsilon-c)N^2\le(1-A_s)\,\varepsilon(1-\varepsilon)N^2$. Every pair of a corrupted example inside and a clean example outside is correctly ordered, which gives $c(1-\varepsilon-\alpha+c)N^2\le A_s\,\varepsilon(1-\varepsilon)N^2$. Dividing by $N^2$ gives the two inequalities.

Both inequalities are attained. For the first, order the groups as corrupted inside, clean inside, corrupted outside, clean outside, with any order within each group. A clean example then outranks a corrupted one only if the clean example is inside and the corrupted one outside, so the mis-ordered pairs are exactly the ones counted, and the first inequality holds with equality. For the second, the order clean inside, corrupted inside, clean outside, corrupted outside makes the correctly ordered pairs exactly the corrupted-inside, clean-outside pairs. Since $A_s$ moves in steps of $1/(\varepsilon(1-\varepsilon)N^2)$, ``attained'' holds up to this rounding.

Solving the two quadratics for $c$ gives the contamination allowed at a given AUROC, the interval reported in \Cref{tab:auroc}:
\[
\frac{\alpha+\varepsilon-\sqrt{(\alpha-\varepsilon)^2+4(1-A_s)\varepsilon(1-\varepsilon)}}{2}\;\le\;c\;\le\;\frac{\sqrt{(1-\varepsilon-\alpha)^2+4A_s\varepsilon(1-\varepsilon)}-(1-\varepsilon-\alpha)}{2},
\]
intersected with the feasible range of (iii). The other roots of the two quadratics lie outside the feasible range and do not constrain $c$.

\emph{(iii) Comparison with random selection.} All four group sizes are non-negative, so $c$ ranges over $\max(0,\alpha+\varepsilon-1)\le c\le\min(\alpha,\varepsilon)$. On this range $f(c)=(\alpha-c)(\varepsilon-c)$ is decreasing, since $f'(c)=2c-\alpha-\varepsilon\le0$ for $c\le\min(\alpha,\varepsilon)$, and $g(c)=c(1-\varepsilon-\alpha+c)$ is increasing, since $g'(c)=1-\varepsilon-\alpha+2c\ge0$ for $c\ge\max(0,\alpha+\varepsilon-1)$. A random subset has $\mathrm{cis}=\varepsilon$ in expectation, i.e.\ $c=\varepsilon\alpha$, where the two functions agree:
\[
f(\varepsilon\alpha)=\alpha(1-\varepsilon)\cdot\varepsilon(1-\alpha)=g(\varepsilon\alpha)=\varepsilon\alpha\cdot(1-\varepsilon)(1-\alpha)=\alpha(1-\alpha)\,\varepsilon(1-\varepsilon).
\]
If $A_s\ge1-\alpha(1-\alpha)$, the first inequality of (ii) gives $f(c)\le(1-A_s)\varepsilon(1-\varepsilon)\le f(\varepsilon\alpha)$, and since $f$ is decreasing, $c\ge\varepsilon\alpha$, i.e.\ $\mathrm{cis}\ge\varepsilon$. If $A_s\le\alpha(1-\alpha)$, the second inequality gives $g(c)\le g(\varepsilon\alpha)$, and since $g$ is increasing, $\mathrm{cis}\le\varepsilon$. If $A_s=1$, the first inequality reads $f(c)\le0$; on the feasible range $f\ge0$ with its only zero at $c=\min(\alpha,\varepsilon)$, so $\mathrm{cis}=\min(1,\varepsilon/\alpha)$.
\end{proof}

\paragraph{Which scores are covered.} Part (i) needs only that the corrupted mean score is at least the clean mean, and parts (ii)--(iii) need only the AUROC, so (ii)--(iii) apply to every score, including \drisShort{} and TRV. The failure statements of the main text concern the common magnitude scores.
\emph{Loss and loss-based IS:} $s=\phi(\ell)$ with $\phi$ strictly increasing preserves the order of the losses, so it has the AUROC of the loss, and (i) applies whenever every corrupted loss exceeds every clean one.
\emph{Mean loss rank (Consensus-loss):} if every corrupted loss exceeds every clean loss under each proxy, every corrupted example has the higher rank under each proxy, and hence the higher mean rank.
\emph{Multi-class EL2N:} with predicted probabilities $p$ and $u=1-p_y=1-e^{-\ell}$, $\|p-e_y\|_2^2=u^2+\sum_{k\ne y}p_k^2$, where the off-target probabilities sum to $u$. Their squares sum to at most $u^2$ and, by Cauchy--Schwarz, to at least $u^2/(C-1)$, so
\[
\sqrt{\tfrac{C}{C-1}}\;u\;\le\;\|p-e_y\|_2\;\le\;\sqrt2\,u .
\]
For $C=2$ the bounds coincide and EL2N $=\sqrt2\,(1-e^{-\ell})$ is a strictly increasing function of the loss. For $C>2$, EL2N preserves the order of two examples whose losses are far enough apart: if $u_i\ge\sqrt{2(C-1)/C}\,u_j$ for all $i\in\Dcorr$ and $j\in\Dclean$, then $\mathrm{EL2N}_i\ge\mathrm{EL2N}_j$. This is the margin condition of the main text.
\emph{GraNd and gradient-norm IS} are not functions of the loss alone. Part (i) applies to them through their measured ratio $x_s$, and (ii)--(iii) through their measured AUROC. For gradient-norm IS under high-norm noise, the realized visit shares $Q=0.718$ ($\varepsilon=0.1$) and $0.834$ ($\varepsilon=0.25$) at $\xi=0.1$ give, by (i), $x_s=7.8$ and $3.6$.

\subsection{Proof of \Cref{prop:ranks}}
\label{app:prop_ranks}
Each proxy's normalized ranks are a permutation of $\{\frac1N,\dots,1\}$, so their moments over the dataset are the same for every proxy. With $m=\frac{N+1}{2N}$ and $v_N=\frac{N^2-1}{12N^2}$,
\[
\frac1N\sum_i\rho_{ik}=m,\qquad \frac1N\sum_i\rho_{ik}^2=m^2+v_N,\qquad \frac1N\sum_i\rho_{ik}\rho_{il}=m^2+r_{kl}v_N,
\]
where the last identity holds because Spearman's correlation $r_{kl}$ is the Pearson correlation of the ranks. We also use $m(1-m)=\frac{N^2-1}{4N^2}=3v_N$.

\begin{proof}
\emph{(a)} Expanding the square in \eqref{eq:rankvar} gives $\rankvar_i=\frac1K\sum_k\rho_{ik}^2-\overline\rho_i^2$, and by definition $W_i=\overline\rho_i-\frac1K\sum_k\rho_{ik}^2$. In the sum the second moments cancel: $\rankvar_i+W_i=\overline\rho_i-\overline\rho_i^2=\widehat b_i$. Each term $\rho_{ik}(1-\rho_{ik})$ of $W_i$ is non-negative because $\rho_{ik}\in[0,1]$.

\emph{(b)} Write $\overline\rho_i^2=\frac1{K^2}\sum_{k,l}\rho_{ik}\rho_{il}$ and average over $i$. The $K$ diagonal terms contribute $m^2+v_N$ each and the $K(K-1)$ off-diagonal terms $m^2+r_{kl}v_N$, so
\[
\rankbar=(m^2+v_N)-\frac1{K^2}\Bigl[K(m^2+v_N)+\sum_{k\ne l}(m^2+r_{kl}v_N)\Bigr] = \Bigl(1-\frac1K\Bigr)v_N-\frac{K-1}{K}\,\bar r\,v_N=\Bigl(1-\frac1K\Bigr)(1-\bar r)\,v_N,
\]
where $\bar r$ averages $r_{kl}$ over the $K(K-1)$ ordered pairs. For the within-proxy part, $\frac1N\sum_iW_i=m-(m^2+v_N)=m(1-m)-v_N=2v_N$.

\emph{(c)} If all $K$ ranks of $i$ lie in $[1-\tau,1]$, so does their mean $\mu=\overline\rho_i$. Popoviciu's inequality bounds the variance of numbers in an interval of length $\tau$ by $\tau^2/4$. The Bhatia--Davis inequality on the same interval gives $\rankvar_i\le(1-\mu)(\mu-1+\tau)\le\tau(1-\mu)$, and since $\widehat b_i=\mu(1-\mu)\ge(1-\tau)(1-\mu)$, this is at most $\frac{\tau}{1-\tau}\widehat b_i$. Finally, $\widehat b_i\le1-\mu\le\tau$.

\emph{(d)} With $K-1$ ranks equal to $r$ and one equal to $r-\delta$, the mean is $\mu=r-\delta/K$. The deviations from the mean are $\delta/K$ ($K-1$ times) and $-(K-1)\delta/K$ (once), so
\[
\rankvar_i=\frac1K\Bigl[(K-1)\frac{\delta^2}{K^2}+\frac{(K-1)^2\delta^2}{K^2}\Bigr]=\frac{K-1}{K^2}\,\delta^2
\]
for every $r$, and $\widehat b_i=\mu(1-\mu)$ by definition. The admission conditions follow by solving $\frac{K-1}{K^2}\delta^2>t$ for $\delta$, and $\mu(1-\mu)>t$, i.e.\ $(\mu-\frac12)^2<\frac14-t$, for $\mu$.

\emph{(e)} The reflection $\rho\mapsto1+\frac1N-\rho$ maps $\{\frac1N,\dots,1\}$ onto itself. It shifts and negates every rank, which leaves the variance across proxies unchanged, and it maps $\overline\rho_i$ to $1+\frac1N-\overline\rho_i$, which changes $\widehat b_i$ by $(2\overline\rho_i-1)/N-1/N^2=O(1/N)$.
\end{proof}

\noindent\emph{Remarks.} On the deployed CIFAR-10 high-norm-$25\%$ setting the proxies' mean pairwise Spearman correlation is $\bar r=0.73$--$0.76$, and (b) gives $\rankbar=0.0135$--$0.0149$ against $\bar b=0.180$--$0.182$: almost all of TRV's mean is the within-proxy constant $2v_N\approx\frac16$. Part (d) describes the most common escape. At $\alpha=0.25$ the selection thresholds are the $75$th percentiles of the scores, $t\approx0.02$ for $\rankvar$ and $t\approx0.24$ for $\widehat b$. With $K=3$, $\rankvar$ then admits a single escape iff $\delta>3\sqrt{0.01}=0.3$, while $\widehat b$ admits an example only if its mean rank lies in $(0.4,0.6)$, which for a corrupted example ranked at the top by the two other proxies needs $\delta>1.2$ and is impossible. For example, ranks $(1,1,0.5)$ give $\rankvar\approx0.056$, above \drisShort{}'s threshold, but $\widehat b\approx0.14$, below TRV's. On the same setting, $23$--$24\%$ of the corrupted examples leave the top $20\%$ of ranks under exactly one proxy. \drisShort{} selects $9$--$11\%$ of them and TRV at most one example, and these escapes account for about two thirds of the gap in subset contamination between the two statistics ($0.072$ vs.\ $0.036$).

\subsection{Proof of \Cref{prop:online}}
\label{app:prop_online}
Write $s_i=\rankvar_i$ and $\bar s=\rankbar>0$, so that $\omega_i=(Nq_i)^{-1}=\frac{(1+\xi)\bar s}{s_i+\xi\bar s}$. The same proof applies to TRV-online with $s=\widehat b$.

\begin{proof}
\emph{(i)} The weights undo the sampling probabilities: $\E_q[g]=\sum_iq_i\,\omega_i\nabla f_i=\frac1N\sum_i\nabla f_i=\nabla\widehat R$.

\emph{(ii)} By \eqref{eq:dris_online}, the weight of example $i$ is
\[
\omega_i=\frac1{Nq_i}=\frac{(1+\xi)\bar s}{s_i+\xi\bar s}.
\]
The numerator does not depend on $i$, and the denominator increases with $s_i$, so $\omega_i$ decreases in $s_i$: the less an example is sampled, the larger its weight. The largest weight therefore belongs to the example with the smallest score,
\[
\max_i\omega_i=\frac{(1+\xi)\bar s}{\min_js_j+\xi\bar s}=\kappa .
\]
Scores are non-negative, so $\min_js_j\ge0$, and dropping it from the denominator can only increase the fraction:
\[
\kappa\le\frac{(1+\xi)\bar s}{\xi\bar s}=1+\frac1\xi .
\]
The bound is attained when some example has score zero. For \drisShort{} this happens whenever all $K$ proxies give an example the same rank. Without smoothing ($\xi=0$), such an example would never be sampled and its weight would be infinite. The smoothing term $\xi\bar s$ keeps every $q_i\ge\frac{\xi}{(1+\xi)N}$ and so caps every weight.

\emph{(iii)} One factor $\omega_i$ cancels against the sampling probability and one remains:
\[
\E_q\|g\|^2=\sum_iq_i\,\omega_i^2\|\nabla f_i\|^2=\frac1N\sum_i\omega_i\|\nabla f_i\|^2=\kappa_{\mathrm{eff}}\,\E\|g_U\|^2\le\kappa\,\E\|g_U\|^2,
\]
using $\E\|g_U\|^2=\frac1N\sum_i\|\nabla f_i\|^2$ and $\omega_i\le\kappa$. For an $L$-smooth $\widehat R$, the descent lemma and (i) give
\[
\E\,\widehat R(w-\eta g)\le\widehat R(w)-\eta\|\nabla\widehat R(w)\|^2+\frac{L\eta^2}{2}\,\kappa_{\mathrm{eff}}\,\E\|g_U\|^2,
\]
and the right-hand side is below $\widehat R(w)$ exactly when $\eta<2\|\nabla\widehat R\|^2/(L\kappa_{\mathrm{eff}}\E\|g_U\|^2)$.
\end{proof}

\noindent In variance form, (iii) gives $\Var_q(g)=\E_q\|g\|^2-\|\nabla\widehat R\|^2\le\kappa\Var(g_U)+(\kappa-1)\|\nabla\widehat R\|^2$, so with the step size scaled by $1/\kappa$, standard SGD rates carry over with the second-moment constant scaled by $\kappa$.

\paragraph{Suppression raises the weights.} Let $\zeta=\sum_{i\in\Dcorr}\|\nabla f_i\|^2/\sum_i\|\nabla f_i\|^2$ be the corrupted examples' share of the squared gradient norm. Keeping only their terms in the numerator of $\kappa_{\mathrm{eff}}$,
\[
\kappa_{\mathrm{eff}}\ge\frac{\sum_{i\in\Dcorr}\omega_i\|\nabla f_i\|^2}{\sum_i\|\nabla f_i\|^2}=\zeta\,\bar\omega^{\nabla}_{\mathrm{corr}},\qquad\text{where}\quad \bar\omega^{\nabla}_{\mathrm{corr}}=\frac{\sum_{i\in\Dcorr}\omega_i\|\nabla f_i\|^2}{\sum_{i\in\Dcorr}\|\nabla f_i\|^2}
\]
is a gradient-weighted mean of the corrupted weights. Their plain mean is bounded by the arithmetic--harmonic mean inequality and \Cref{prop:suboptimality}(i):
\[
\frac1{\varepsilon N}\sum_{i\in\Dcorr}\omega_i\;\ge\;\Bigl(\frac1{\varepsilon N}\sum_{i\in\Dcorr}\omega_i^{-1}\Bigr)^{-1}=\frac{\varepsilon}{Q}.
\]
If the weights and the squared gradient norms are uncorrelated on $\Dcorr$, the two means coincide and $\kappa_{\mathrm{eff}}\ge\zeta\,\varepsilon/Q$: the more the sampler suppresses visits to the corrupted examples, the larger their weights and the effective second-moment constant.

\noindent\emph{Remark.} If the gradient norms are uncorrelated with the weights over all of $\D$, as early in training, $\kappa_{\mathrm{eff}}$ equals the average weight $\bar\omega=\frac1N\sum_i\omega_i$. By the same inequality $\bar\omega\ge1$, since $\frac1N\sum_i\omega_i^{-1}=\sum_iq_i=1$. It grows as $\xi$ shrinks and as more examples sit near zero score; on the deployed setting $\bar\omega=2.8$ for \drisShort{} and $1.2$ for TRV. The worst case $\kappa$ is pessimistic, because the largest weights fall on examples with small $\rankvar$, which include easy clean examples with small gradients.

\section{Large-Ensemble Guarantees}
\label{app:guarantees}
This section proves \Cref{thm:main} and gives the corresponding guarantees for TRV. Besides \Cref{assu:adv_concentration,assu:bdry_disagreement}, the contamination bound needs a budget on how much the escaping tail varies across proxies.
\begin{assumption}[Tail variance budget]\label{assu:proxy_sub_gaus}
$\frac{1}{|\Dtail|}\sum_{i\in\Dtail}\Var_{h\sim\mathcal P}[\rho_i(h)]\le v_{\mathrm{tail}}$.
\end{assumption}
\noindent Ranks lie in $[0,1]$, so Popoviciu's inequality gives $v_{\mathrm{tail}}\le\frac14$ for any data; the assumption matters only for smaller values. As in \Cref{thm:main}, $\rK=\sqrt{\log(2N/\delta)/(2K)}$.

\subsection{Proof of \Cref{thm:main}}
\label{app:proof}
The proof has four steps. All scores concentrate around their expectations at once (Step 1); the expectation is small on the corrupted bulk (Step 2) and large on the clean boundary (Step 3); a gap between the two keeps the bulk out of $S_\alpha$, and only part of the tail can enter (Step 4).

\begin{proof}
\textbf{Step 1: uniform concentration.} Fix $i$ and write $\rho_k=\rho_i(h_k)\in[0,1]$, so $\rankvar_i=\frac1K\sum_k\rho_k^2-\bar\rho^2$ is a function of the $K$ independent proxies. Replacing one rank $a$ by $b$, with $S$ the sum of the other $K-1$ ranks, changes it by
\[
\frac{b^2-a^2}{K}-\frac{(S+b)^2-(S+a)^2}{K^2}=\frac{(b-a)\bigl[(K-1)(a+b)-2S\bigr]}{K^2}.
\]
Since $0\le S\le K-1$, the bracket lies between $-(K-1)(2-a-b)$ and $(K-1)(a+b)$, and both $|b-a|(a+b)=|b^2-a^2|$ and $|b-a|(2-a-b)=|(1-a)^2-(1-b)^2|$ are at most $1$. One substitution therefore changes $\rankvar_i$ by at most $(K-1)/K^2<1/K$. McDiarmid's inequality \citep{mcdiarmid1989method} with $c_k=1/K$ gives $\Prob(|\rankvar_i-\E\rankvar_i|\ge t)\le2\exp(-2Kt^2)$. The choice $t=\rK$ makes the right-hand side $\delta/N$, and a union bound over the $N$ examples gives
\begin{equation}\label{eq:star}
\Prob\bigl(|\rankvar_i-\E\rankvar_i|\le\rK\ \text{for all } i\bigr)\ge1-\delta.
\end{equation}
The remaining steps hold on this event.

\smallskip\noindent\textbf{Step 2: small expected disagreement on the bulk.} For i.i.d.\ proxies, the biased sample variance satisfies $\E\rankvar_i=(1-\frac1K)\Var_{\mathcal P}[\rho_i(h)]$. Fix $i\in\Dbulk$, let $X=\rho_i(h)$ for $h\sim\mathcal P$, and let $\mathcal A=\{X\ge1-\tau\}$, so that $\Prob(\mathcal A)\ge1-\gamma$ by \Cref{assu:adv_concentration}. The variance is at most the mean squared deviation from any constant, in particular from the midpoint $1-\frac\tau2$ of $[1-\tau,1]$. This deviation is at most $\frac\tau2$ on $\mathcal A$ and at most $1$ on $\mathcal A^c$, so
\[
\Var(X)\le\E\bigl[(X-1+\tfrac\tau2)^2\bigr]\le\tfrac{\tau^2}{4}\Prob(\mathcal A)+\Prob(\mathcal A^c)\le\tfrac{\tau^2}{4}+\gamma.
\]
With \eqref{eq:star}, every $i\in\Dbulk$ has $\rankvar_i\le\theta^*:=(1-\frac1K)(\frac{\tau^2}4+\gamma)+\rK$.

\smallskip\noindent\textbf{Step 3: large expected disagreement on the boundary.} For $j\in\Dbdry$, \Cref{assu:bdry_disagreement} gives $\E\rankvar_j=(1-\frac1K)\Var_{\mathcal P}[\rho_j]\ge(1-\frac1K)\tau^2_{\mathrm{bdry}}$, and with \eqref{eq:star}, $\rankvar_j\ge(1-\frac1K)\tau^2_{\mathrm{bdry}}-\rK$.

\smallskip\noindent\textbf{Step 4: separation and contamination.} The condition $\Delta'>2\rK$ says exactly that the boundary floor of Step 3 exceeds the bulk ceiling $\theta^*$ of Step 2. Every boundary example then scores above $\theta^*$, and every bulk example at or below it. If moreover $|\Dbdry|\ge\alpha N$, at least $\alpha N$ examples score above $\theta^*$, so all members of $S_\alpha$ do, and $S_\alpha\cap\Dbulk=\emptyset$.

The only corrupted examples that can enter $S_\alpha$ are therefore in the tail. Such an example has $\rankvar_i>\theta^*$, so by \eqref{eq:star} $\E\rankvar_i>(1-\frac1K)(\frac{\tau^2}4+\gamma)$, that is,
\begin{equation}\label{eq:tail_survival_condition}
\Var_{h\sim\mathcal P}[\rho_i(h)]>\tfrac{\tau^2}{4}+\gamma.
\end{equation}
By \Cref{assu:proxy_sub_gaus} the tail's variances average at most $v_{\mathrm{tail}}$, so Markov's inequality, applied to the variance of a uniformly drawn tail example, shows that at most a fraction $\min\bigl(1,v_{\mathrm{tail}}/(\frac{\tau^2}4+\gamma)\bigr)$ of the tail satisfies \eqref{eq:tail_survival_condition}. This is a deterministic count over the tail, not a further probability. Hence
\[
|\Dcorr\cap S_\alpha|\le|\Dtail|\min\Bigl(1,\frac{v_{\mathrm{tail}}}{\tau^2/4+\gamma}\Bigr)=\atrim\varepsilon N\min\Bigl(1,\frac{v_{\mathrm{tail}}}{\tau^2/4+\gamma}\Bigr),
\]
and dividing by $|S_\alpha|=\alpha N$ gives \eqref{eq:contam}.
\end{proof}

\paragraph{Remark (how many proxies the theorem needs).} Ignoring the factor $1-\frac1K$, the condition $\Delta'>2\rK$ reads $K>2\log(2N/\delta)/\widetilde\Delta^2$ with $\widetilde\Delta=\tau^2_{\mathrm{bdry}}-\tau^2/4-\gamma$. Popoviciu's inequality gives $\tau^2_{\mathrm{bdry}}\le\frac14$, so $\widetilde\Delta\le\frac14$, and even in the most favourable case $K>32\log(2N/\delta)$, i.e.\ $K\ge465$ for $N=5\cdot10^4$ and $\delta=0.05$. At the measured boundary variances ($\widehat\tau^2_{\mathrm{bdry}}\approx0.0085$--$0.0145$, \Cref{tab:proxy_diversity}), even $\tau=\gamma=0$ would require $K>2\cdot14.5/0.0145^2\approx1.4\cdot10^5$, and with the measured escape rate $\widehat\gamma\approx0.2$--$0.3$, $\widetilde\Delta<0$ and the condition fails for every $K$. At $K=3$, $\rK\approx1.55$ exceeds $\frac14$, the largest possible variance of a $[0,1]$-valued variable, so \Cref{thm:main} is vacuous for the deployed ensemble. The deployed ensemble is certified instead by \Cref{prop:suboptimality}(iii) at its measured AUROC, and that $K=3$ suffices in practice is an empirical finding (\Cref{sec:exp:checks}).

\paragraph{Remark (non-identical proxies).} Step 1 uses only the independence of the proxies. If $h_k\sim\mathcal P_k$ are independent but not identically distributed, for example of different architectures, with $\mu_{ik}=\E\rho_i(h_k)$ and $v_{ik}=\Var\rho_i(h_k)$, then
\[
\E\rankvar_i=\Bigl(1-\frac1K\Bigr)\frac1K\sum_kv_{ik}+\frac1K\sum_k\Bigl(\mu_{ik}-\frac1K\sum_l\mu_{il}\Bigr)^2 .
\]
The first term replaces $(1-\frac1K)\Var_{\mathcal P}[\rho_i]$, and the second is the spread of the mean ranks across architectures. For a bulk-corrupted example ranked near the top under every architecture both terms are small, so Steps 2--4 go through with $\Var_{\mathcal P}$ replaced by the average per-architecture variance plus this spread.

\subsection{Guarantees for TRV}
\label{app:prop_band}
TRV, $\widehat b_i=\overline\rho_i(1-\overline\rho_i)$ with $\overline\rho_i=\frac1K\sum_k\rho_i(h_k)$, is the total variance of which $\rankvar_i$ is the between-proxy part (\Cref{prop:ranks}(a)), so $\rankvar_i\le\widehat b_i$. It is not monotone in the loss, since it falls once $\overline\rho_i>\frac12$, so the failure of loss-monotone scores does not apply to it; \Cref{prop:suboptimality} still applies through its measured $x$ and AUROC. Its large-ensemble analysis is simpler than that of \drisShort{}, because $\widehat b_i$ is a smooth function of a mean rather than a variance. Let $\mu_i=\E\overline\rho_i$ and $b_i=\mu_i(1-\mu_i)$.

\begin{proposition}[TRV guarantees]\label{prop:band}
Let the proxies be independent, with \Cref{assu:adv_concentration} holding for each proxy's distribution, and let $\rK=\sqrt{\log(2N/\delta)/(2K)}$ as in \Cref{thm:main}.
\begin{enumerate}[label=(\roman*),nosep,leftmargin=*]
\item \textbf{Concentration:} with probability at least $1-\delta$, $|\widehat b_i-b_i|\le\rK$ for all $i$ simultaneously.
\item \textbf{Bulk corrupted:} $b_i\le1-\mu_i\le\tau+\gamma$ for all $i\in\Dbulk$; on the event of (i), $\widehat b_i\le\tau+\gamma+\rK$.
\item \textbf{Mid-difficulty clean:} if $\mathcal D_{\mathrm{mid}}\subseteq\Dclean$ satisfies $\mu_j\in[m,1-m]$ for some $m\in(0,\frac12]$ and all $j\in\mathcal D_{\mathrm{mid}}$, then $b_j\ge m(1-m)$ and, on the event of (i), $\widehat b_j\ge m(1-m)-\rK$.
\item \textbf{Separation:} if $m(1-m)-\tau-\gamma>2\rK$ and $|\mathcal D_{\mathrm{mid}}|\ge\alpha N$, the top-$\alpha N$ set by $\widehat b$ contains no bulk-corrupted example, and its contamination is at most $\atrim\varepsilon/\alpha$.
\item \textbf{Clean-side domination:} for every $j$, $\Var_{\mathcal P}[\rho_j]\le\mu_j(1-\mu_j)$ when the proxies are i.i.d.; hence every $\Dbdry$ satisfying \Cref{assu:bdry_disagreement} satisfies the hypothesis of (iii) with $m(1-m)=\tau^2_{\mathrm{bdry}}$.
\end{enumerate}
\end{proposition}

\begin{proof}
\emph{(i)} $\overline\rho_i$ is the mean of $K$ independent $[0,1]$-valued variables, so Hoeffding's inequality gives $\Prob(|\overline\rho_i-\mu_i|>t)\le2e^{-2Kt^2}$, and a union bound over the $N$ examples with $t=\rK$ gives $|\overline\rho_i-\mu_i|\le\rK$ for all $i$ with probability at least $1-\delta$. The map $x\mapsto x(1-x)$ is $1$-Lipschitz on $[0,1]$, so $|\widehat b_i-b_i|\le\rK$. Unlike in \Cref{thm:main}, no bounded-differences argument for a variance is needed.

\emph{(ii)} For $i\in\Dbulk$ and each proxy, $\E\rho_i(h_k)\ge(1-\tau)\Prob(\rho_i(h_k)\ge1-\tau)\ge(1-\tau)(1-\gamma)$, so $\mu_i\ge(1-\tau)(1-\gamma)\ge1-\tau-\gamma$. Since $\mu_i\le1$, $b_i=\mu_i(1-\mu_i)\le1-\mu_i\le\tau+\gamma$; the bound on $\widehat b_i$ follows from (i).

\emph{(iii)} On $[m,1-m]$, $x(1-x)$ is smallest at the endpoints, where it equals $m(1-m)$; the bound on $\widehat b_j$ follows from (i).

\emph{(iv)} On the event of (i), the condition places every bulk score below every mid-difficulty score, so the top $\alpha N$ positions are filled from $\mathcal D_{\mathrm{mid}}$ before any bulk example. Only the tail, of size $\atrim\varepsilon N$, can enter.

\emph{(v)} Since $\rho_j(h)\in[0,1]$, the Bhatia--Davis inequality gives $\Var_{\mathcal P}[\rho_j]\le\mu_j(1-\mu_j)$. If $\Var_{\mathcal P}[\rho_j]\ge\tau^2_{\mathrm{bdry}}$, then $\mu_j(1-\mu_j)\ge\tau^2_{\mathrm{bdry}}=m(1-m)$, i.e.\ $\mu_j\in[m,1-m]$.
\end{proof}

\paragraph{Remarks.} (i) The two scores trade a tighter bulk ceiling for a weaker requirement on the clean side. \drisShort{}'s bulk ceiling is $(1-\frac1K)(\frac{\tau^2}4+\gamma)$ and TRV's is $\tau+\gamma$ (more precisely $\tau+\gamma-\tau\gamma$), while by (v) TRV's clean-side hypothesis is implied by \drisShort{}'s. The raw ceilings are not comparable across scores: a sampler normalizes each score by its average, and TRV's average contains the constant $2v_N\approx\frac16$ (\Cref{prop:ranks}(b)); the resulting comparison of the two samplers is in \Cref{sec:method}. For static selection, a corrupted example with a single escaped rank enters \drisShort{}'s subset but not TRV's (\Cref{app:prop_ranks}, remarks). (ii) $K$ enters TRV only through the concentration of $\overline\rho_i$; with $K=1$ the score is a noisy single-proxy rank, and empirically TRV from a single proxy is clearly weaker than from $K=3$ (\Cref{app:exp:band}). (iii) As for \Cref{thm:main}, the separation (iv) needs very large $K$: since $m(1-m)\le\frac14$, it requires $\tau+\gamma<\frac14$, which the measured escape rates violate.

\section{Experiment Details} \label{app:exp_details}

\subsection{Datasets and noise injection} \label{app:noise_injection}
CIFAR-10/100 \citep{krizhevsky2009learning} use the standard $50{,}000/10{,}000$ splits. Food-101 \citep{bossard2014food} uses the standard $750/250$ images per class. \textbf{High-norm noise.} A noise model (ResNet-20 for CIFAR trained for $20$ epochs; ResNet-18 briefly for Food-101) is trained on clean data. The top-$\varepsilon$ fraction of examples by gradient norm under this model has its labels flipped to a uniformly drawn other class. The noise model serves \emph{only} to build the corruption mask; \drisShort{}, all baselines, and all proxies see only the corrupted dataset. \textbf{Other noise models.} The same construction with a VGG-16-BN or a CIFAR-adapted ResNet-18 noise model; their top-$10\%$ gradient-norm sets overlap the ResNet-20 one by $28\%$ and $26\%$. \textbf{Volatile-margin noise.} A ResNet-20 is trained on clean CIFAR-10 with checkpoints every 2 epochs, and the $25\%$ of examples with the largest standard deviation of the margin across the 10 checkpoints are flipped; this set overlaps the high-norm one by $12\%$. \textbf{Uniform noise.} A random $\varepsilon$-fraction of labels is flipped to a uniformly drawn other class. Within a configuration, all methods use the same mask. \textbf{CIFAR-10N.} The human labels of \citet{wei2022learning} (aggre and worst label sets); the mask marks the examples whose human label differs from the original one (\Cref{app:cifar10n}).

\subsection{Proxy ensemble}
\label{app:proxy_details}
\textbf{CIFAR-10/100.} $K=3$ models from \{ResNet-20, MobileNetV2$_{0.5\times}$, ShuffleNetV2$_{0.5\times}$\} \citep{sandler2018mobilenetv2,ma2018shufflenetv2}, each trained for $40$ epochs with batch $128$, SGD with momentum $0.9$, weight decay $5\cdot10^{-4}$, a cosine schedule from $0.05$ \citep{loshchilov2017sgdr}, and mixed precision. Losses for $\rankvar$ are recorded at epoch $20$ (the cost estimates of \Cref{app:cost} assume training stops there). Training continues to epoch $40$ only to supply the margin and forgetting trajectories for the AUM and Forgetting baselines, which share these runs. \textbf{Food-101.} $K=3$ ImageNet-pretrained backbones (ResNet-18, MobileNetV3-Small, EfficientNet-B0). Penultimate features are cached once, and a linear classifier is trained on them for $6$ epochs (batch $4096$, learning rate $0.01$, momentum $0.9$, weight decay $10^{-4}$, cosine schedule). Losses are taken from the final classifier after the $6$th epoch.

\textbf{Parameters.} The ensemble size $K$ is typically $3$--$16$; larger $K$ tightens concentration but helps only before the onset of memorization (\Cref{app:exp:capacity}). The smoothing $\xi$ controls the weight bound $\kappa\le1+1/\xi$; we use $\xi=0.1$ throughout. Ranks are percentiles in $(0,1]$ of each proxy's per-example losses at the scoring epoch; the static rule keeps the top $\alpha N$ examples by the score and trains on them for $1/\alpha$ times the full-data epochs.

\subsection{Seeds and identical runs}
\label{app:seeding}
In every run, the proxies use fixed seeds ($100$, $101$, $102$). The run seed sets the noise draw (which examples are flipped under uniform noise, and the replacement class under both noise types), the random subset, and the target's initialization and data order. The reported standard deviations therefore measure target-training and noise-draw variability, not proxy variability. Selections recomputed from proxy caches match the recorded runs to within environment noise.

\subsection{Target training}
\label{app:target_training}
\textbf{CIFAR-10:} VGG-19-BN \citep{simonyan2015very}, SGD with momentum $0.9$, weight decay $5\cdot10^{-4}$, cosine schedule, learning rate $0.1$, batch $128$; at $\alpha=0.25$, $640$ subset epochs ($=160$ full-data epochs). \textbf{CIFAR-100:} ResNet-18 \citep{he2016deep} adapted to $32\times32$ inputs, same optimizer, equivalent of $200$ full-data epochs. \textbf{Food-101:} ResNet-50 from ImageNet weights, with the last block and classifier fine-tuned; SGD with momentum $0.9$, weight decay $10^{-4}$, learning rate $0.01$, batch $64$; $\alpha=0.5$; the equivalent of $50$ (E1) and $25$ (E2) full-data epochs; three seeds. We halve the E2 budget deliberately, because the effect of high-norm noise appears early.

\subsection{Baselines}
\label{app:baselines}
\textbf{AUM.} The original recipe removes the bottom fraction by AUM; we keep the top-$\alpha$ fraction, which is the same operation under set complement. \textbf{EL2N and Forgetting} are averaged over the same $K$ proxies.
 \textbf{GraNd} \citep{paul2021deep} is a magnitude snapshot that \Cref{prop:suboptimality} covers, like EL2N and gradient-norm IS, so we omit it. \textbf{Gradient-norm IS.} Exact per-sample gradient norms of a VGG-19-BN (the target architecture) trained for $20$ epochs on the corrupted data, computed once before training; sampling $q_i\propto\|\nabla f_i\|+\xi\cdot\mathrm{mean}_j\|\nabla f_j\|$ with the same $\xi$ and weights as \drisShort{}-online. \textbf{RHO-LOSS score, holdout reference.} A ResNet-20 reference model is trained on a $5{,}000$-example holdout for $20$ epochs; the score $\max(\mathcal L(x;w_{\mathrm{target}})-\mathcal L(x;w_{\mathrm{ref}}),0)$ is recomputed every $5$ epochs. \textbf{DivideMix.} We wrap the official code around our data, models, and batch size and leave its other hyperparameters unchanged. \textbf{GMM-select} averages DivideMix's per-proxy two-component loss-GMM clean posteriors over the same $K=3$ proxies and keeps the top $\alpha N$. \textbf{TRV.} Keeps the top $\alpha N$ by $\overline\rho_i(1-\overline\rho_i)$ from the same $K=3$ proxy ranks as \drisShort{}; TRV-online samples with \eqref{eq:dris_online} with $\widehat b_i$ in place of $\rankvar_i$.

\subsection{Smoothing constant}
\label{app:smoothing}
We use $\xi=0.1$ ($10\%$ of the average score added uniformly) for every IS sampler: \drisShort{}-online, gradient-norm IS, and both RHO-LOSS variants. For all of them, the corresponding worst-case weight bound is $\kappa\le11$ (\Cref{prop:online}). Each sampled loss is multiplied by its exact weight $(Nq_i)^{-1}$ and the mini-batch mean is taken, without normalization or clipping. \textbf{Hyperparameter protocol.} $K=3$, the proxy architectures, the scoring epoch $20$, and $\xi=0.1$ were fixed before the first experiment; none was tuned on the corruption mask or on test accuracy. A later sweep (\Cref{app:online_ablations}) finds $\xi=0.1$ to be the best of four values at high-norm-$25\%$.

\subsection{Compute}
\label{app:compute}
Every configuration ran on a single A100-class GPU. Together, all reported experiments took about $16$ GPU-days.

\subsection{Cost scaling}
\label{app:cost}
Static \drisShort{} keeps optimizer-step parity, so the target cost equals that of full-data training. The only extra compute is the one-time proxy ensemble: proxy training plus one forward-only scoring pass. With training at $6\times$MACs per sample and scoring at $2\times$MACs,
\[
\frac{C_{\mathrm{proxy}}}{C_{\mathrm{target}}}\approx\frac{T_{\mathrm{proxy}}}{T_{\mathrm{target}}}\cdot\frac{\sum_k\mathrm{MACs}_k}{\mathrm{MACs}_{\mathrm{target}}},
\]
which does not depend on $N$. Beyond ImageNet-1k, we budget proxy training in steps, not epochs (capped at $51.2$M samples), so only the scoring pass scales with $N$. \Cref{tab:cost_scaling} reports the result. These are FLOP counts; wall-clock time follows them closely but not exactly. Two claims should be kept apart. \emph{Cost} is settled analytically by step parity. \emph{Transfer}, meaning whether \Cref{assu:adv_concentration} holds for ImageNet-scale proxies, is empirical and untested here.

\begin{table}[h]
\centering
\caption{Analytic proxy overhead $C_{\mathrm{proxy}}/C_{\mathrm{target}}$ in FLOPs ($K=3$; training $6\times$MACs/sample, scoring $2\times$MACs/sample). Proxy training is budgeted at $20$ epochs, capped at $51.2$M samples ($^\dagger$) beyond ImageNet-1k. $T_{\mathrm{proxy}}=40$ (our native schedule, kept only to supply trajectories to AUM/Forgetting) doubles the training term: $2.10\to4.16\%$, $0.86\to1.70\%$, $0.40\to0.78\%$, $0.72\to1.42\%$; capped rows are unchanged.}
\label{tab:cost_scaling}
\scriptsize
\setlength{\tabcolsep}{3pt}
\resizebox{\linewidth}{!}{\begin{tabular}{lllllll}
\toprule
Dataset ($N$) & Target (params / MACs / epochs) & Proxies & $\sum_k\mathrm{MACs}_k$ & Proxy budget & $C_{\mathrm{target}}$ / $C_{\mathrm{proxy}}$ (train+score) & Overhead \\
\midrule
CIFAR-10 (50k) & VGG-19-BN / 20.0M / 0.40G / 160 & ResNet-20, MNv2-0.5$\times$, SNv2-0.5$\times$ & 0.066G & 20 ep (1.0M) & 19.2 PF / 0.40 PF (0.40+0.01) & \textbf{2.10\%} \\
ImageNet-1k (1.28M) & ResNet-50 / 25.6M / 4.11G / 90 & MNv3-S, SNv2-0.5$\times$, MNv2-0.35$\times$ & 0.156G & 20 ep (25.6M) & 2.84 EF / 24.4 PF (24.0+0.40) & \textbf{0.86\%} \\
ImageNet-1k (1.28M) & ConvNeXt-T / 28.6M / 4.47G / 300 & MNv3-S, SNv2-1.0$\times$, MNv2-0.35$\times$ & 0.261G & 20 ep (25.6M) & 10.3 EF / 40.8 PF (40.1+0.67) & \textbf{0.40\%} \\
ImageNet-1k (1.28M) & ViT-B/16 / 86.6M / 17.6G / 300 & DeiT-Ti, EffNet-B0, MNv3-L & 1.867G & 20 ep (25.6M) & 40.6 EF / 292 PF (287+4.78) & \textbf{0.72\%} \\
ImageNet-21k (14.2M) & ViT-B/16 / 86.6M / 17.6G / 90 & DeiT-Ti, EffNet-B0, MNv3-L & 1.867G & 51.2M$^\dagger$ (3.6 ep) & 135 EF / 627 PF (574+53.0) & \textbf{0.46\%} \\
JFT-scale (303M) & ViT-L/16 / 304M / 61.6G / 7 & DeiT-Ti, EffNet-B0, MNv3-L & 1.867G & 51.2M$^\dagger$ (0.17 ep) & 784 EF / 1.70 EF (0.57+1.13) & \textbf{0.22\%}
 \\
\bottomrule
\end{tabular}}
\end{table}

\subsection{Software and licenses}
\label{app:license}
PyTorch \citep{paszke2019pytorch} and TorchVision (BSD), NumPy \citep{harris2020array}, SciPy \citep{virtanen2020scipy}, Pandas \citep{mckinney2010data}, scikit-learn \citep{pedregosa2011scikit} (BSD), Matplotlib \citep{hunter2007matplotlib} and Seaborn (BSD), \texttt{tqdm} (MIT/MPL), \texttt{PyYAML} (MIT), and \texttt{OmegaConf} (BSD), each used under its license terms. The datasets (CIFAR-10/100, CIFAR-10N, Food-101) are public research benchmarks, used under their terms of use; the CIFAR-10N labels are licensed CC BY-NC 4.0.

\subsection{Threat model}
\label{app:threat_model}
\emph{Adversary knowledge:} the training set and, for high-norm and volatile-margin flips, a surrogate model trained on clean data (\Cref{app:noise_injection}); no knowledge of the realized proxies $h_1,\dots,h_K$, since the corruption is fixed before they are drawn (\Cref{assu:contamination}).

\emph{Adversary capability:} before training, change the labels of any $\varepsilon N$ training examples, $\varepsilon<\frac12$, chosen as a function of the data alone; this is strong contamination \citep{huber1964robust,diakonikolas2019recent} restricted to labels.

\emph{Adversary objective:} availability, i.e.\ degrading the test accuracy of the trained target, in particular by exploiting how magnitude-based selection allocates visits and subset slots (\Cref{prop:suboptimality}).

\emph{Defender:} observes only the corrupted dataset and uses no clean auxiliary data; it trains the $K$ proxies on this data and acts only by selection, keeping a subset (pruning) or setting visit probabilities and weights (importance sampling), never by relabeling.

Uniform flips and human label noise (\Cref{app:cifar10n}) are the special case of a non-strategic adversary. Volatile-margin flips, which target the examples whose margin is least stable under the surrogate, are a first-order attempt to evade rank disagreement (\Cref{app:noise_models}). An adversary with access to the realized proxies, or one that optimizes through the proxy training distribution, is outside this model and untested.

The model fits pipelines in which labels come from sources that cannot be audited, such as web-scale curation or crowdsourced and continual labeling, and no trusted clean subset exists to calibrate a filter.

\section{Online Importance Sampling: Mechanism and Ablations}
\label{app:online_ablations}
All online runs train a VGG-19-BN on all $N$ examples of CIFAR-10 for $160$ epochs, with batch $128$, a cosine schedule from learning rate $0.1$ without warm-up or clipping, and $3$ seeds. The sampling scores are computed once from the $K=3$ proxies at their epoch $20$ and stay fixed during training. Every run logs, every $5$ epochs, how many corrupted examples the target fits. Unless stated otherwise, $\xi=0.1$.

\Cref{tab:visits_fit} is the visit ledger at high-norm-$25\%$. For every sampler it lists the visit share that \Cref{prop:suboptimality}(i) predicts from the measured score ratio $x$, the share realized during training, the fraction of corrupted labels fit at the end, and the accuracy; \drisShort{} appears at four smoothings. For TRV-online and for \drisShort{} at all four smoothings, the predicted and realized shares agree to within $0.005$, so the first link of the mechanism, from the score to the visits, holds exactly.

\emph{Smoothing.} $\xi$ interpolates between pure score sampling and uniform SGD. At $\xi=1$, half the sampling mass is uniform, and accuracy ($76.1\%$) approaches uniform SGD's ($74.9\%$), while the fit ($0.58$) stays well below uniform SGD's $0.94$. At $\xi=0.3$ the gain over uniform SGD is $6.0$\,pp, with no divergence. At $\xi=0.03$ the weights reach $34$, one seed collapses, and a second ends at $59.5\%$. The default $\xi=0.1$, fixed before any of these runs, is the best of the four values. A smaller $\xi$ suppresses visits to the corrupted examples more but makes training less stable; \Cref{app:stability} quantifies this and shows that a learning-rate warm-up removes the instability.
\begin{table}[t]
\centering
\caption{Visits predicted by \Cref{prop:suboptimality}(i) from the measured score ratio $x$ (which uses the corruption mask), the realized visit share, and the outcome, CIFAR-10 high-norm-$25\%$. All samplers use weights $(Nq_i)^{-1}$; gradient-norm IS and TRV use $\xi=0.1$. Fit: corrupted labels predicted as given at the end. Cap: weight bound $1+1/\xi$. $^{(2/3)}$Two seeds; the third diverged. $^{\dagger}$From three separate runs with fit logging: two end at chance, one at $53.5\%$.}
\label{tab:visits_fit}
\scriptsize
\setlength{\tabcolsep}{2.6pt}
\begin{tabular}{l ccc cccc}
\toprule
& & grad- & & \multicolumn{4}{c}{\drisShort{}, $\xi=$} \\
\cmidrule(lr){5-8}
& unif. & norm & TRV & $1$ & $0.3$ & $0.1$ & $0.03$ \\
\midrule
$Q/\varepsilon$ pred. & $1$ & $3.35$ & $.719$ & $.722$ & $.572$ & $.495$ & $.460$ \\
$Q/\varepsilon$ real. & $1$ & $3.33$ & $.719$ & $.722$ & $.572$ & $.495$ & $.455^{(2/3)}$ \\
Fit & $.94$ & $.09^{\dagger}$ & $.62$ & $.58$ & $.31$ & $.15$ & $.07^{(2/3)}$ \\
Acc.\ (\%) & $74.9$ & $10.0$ & $76.0$ & $76.1$ & $80.9$ & $\mathbf{84.5}$ & $71.0^{(2/3)}$ \\
Cap & $1$ & $11$ & $11$ & $2$ & $4.3$ & $11$ & $34$ \\
Diverged & 0/3 & 3/3 & 0/3 & 0/3 & 0/3 & 0/3 & 1/3
 \\
\bottomrule
\end{tabular}
\end{table}

\paragraph{Why an unbiased sampler can be more robust.}
\label{app:online_mechanism}
By \Cref{prop:online}(i), online \drisShort{} and uniform SGD have the same expected update. The $9.6$\,pp gap in \Cref{tab:online_is} must therefore be a finite-horizon effect. We state our explanation as a hypothesis, not a theorem. Over-parameterized networks fit the dominant structure of the data before the noise \citep{li2020gradient}, and memorizing an individual flipped label requires repeated exposure \citep{arpit2017closer}. Under online \drisShort{}, a bulk-corrupted example with $\rankvar_i\approx0$ is drawn at rate $Nq_i\approx\xi/(1+\xi)$ relative to uniform sampling (about $0.09$ for $\xi=0.1$), with weight $(Nq_i)^{-1}$. Its cumulative expected contribution is unchanged, but it arrives in roughly $11\times$ fewer, $11\times$ larger steps, interleaved with many more steps on clean boundary examples. Averaged over the whole corrupted set, which includes the escaping tail, the measured visit rate is a third to a half of uniform ($Q=0.036$ at $\varepsilon=0.10$ and $0.124$ at $\varepsilon=0.25$). If memorization is driven by the number of exposures in the early phase of training rather than by their total weight, these sparse exposures are absorbed by the clean signal before the memorization phase begins.

\Cref{fig:memorization} tests this prediction by logging, every $5$ epochs, how many corrupted examples the target predicts as their (wrong) given label. Under uniform SGD this fraction rises to $0.94$ by the end of training at high-norm-$25\%$ ($0.97$ at $10\%$); under online \drisShort{} it stays at $0.15$ ($0.17$). TRV-online ends at $0.62$ ($0.55$): relative to its average, which contains the constant $2v_N\approx\frac16$, TRV leaves more sampling weight on the corrupted examples (\Cref{prop:ranks}). The static rule does not need this argument, because it removes the corrupted bulk from the objective.
\begin{figure}[h]
\centering
\includegraphics[width=0.8\linewidth]{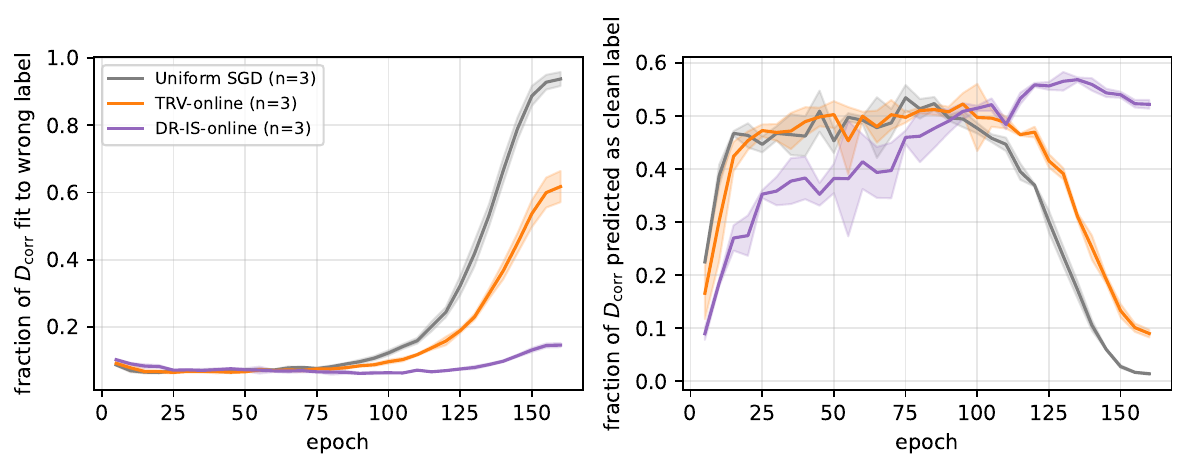}
\caption{Memorization of the corrupted labels during online training (CIFAR-10/VGG-19-BN, high-norm-$25\%$, all $N$ examples, $160$ epochs, mean$\pm$std over 3 seeds). Left: fraction of $\Dcorr$ predicted as its (wrong) given label. Right: fraction predicted as its original label. Evaluated every $5$ epochs on un-augmented images.}
\label{fig:memorization}
\end{figure}

\emph{Why keep the estimate unbiased.} Unbiasedness is a choice, not a requirement: the sampler stays unbiased for $\widehat R$, the very risk that contamination makes the wrong target (\Cref{sec:setting}). We keep it for three reasons. First, it isolates the effect of visits: because the expected update equals that of uniform SGD, any gain comes from how often examples are visited, not from a changed objective. Changing the objective is what the static rule does, by restricting the risk to the kept subset. Second, it bounds the downside: the sampler keeps the stationary points and convergence guarantees of uniform SGD (\Cref{prop:online}), so when the corruption is absent or not loss-aligned the cost stays small ($1.2$\,pp on clean data). Third, it keeps the comparison with gradient-norm IS, which is also unbiased, like for like. The price is the large weights on rarely visited examples, which limit the usable step size (\Cref{app:stability}). Dropping the weights would give a biased sampler that also down-weights the corrupted examples in expectation; we have not evaluated this alternative.

\section{Magnitude-Based Online Samplers: Gradient-Norm IS and RHO-LOSS Scores}
\label{app:rho_loss_breakdown}
RHO-LOSS \citep{mindermann2022prioritized} scores $\mathcal L(x;w_{\mathrm{target}})-\mathcal L(x;w_{\mathrm{ref}})$, with $w_{\mathrm{ref}}$ trained on held-out data. We use the RHO-LOSS \emph{score} inside the same smoothed, importance-weighted sampler as the other online methods ($\xi=0.1$, weights $(Nq_i)^{-1}$). The original protocol instead selects the top-scoring examples of a large candidate batch, without importance weights, so our results concern this use of the score, not the published method. We test two references. With the \emph{proxy reference}, the mean proxy loss replaces the target term, and the score is computed once. With the \emph{holdout reference}, a ResNet-20 is trained on a $5{,}000$-example holdout, and the score is refreshed every $5$ epochs from the current target.

All three magnitude-based samplers underperform uniform sampling in every setting, including clean data (\Cref{tab:magnitude_online}). Gradient-norm IS fails in the way \Cref{prop:suboptimality}(i) predicts, sending most of its visits to the corrupted examples. The RHO-LOSS scores fail differently: the per-seed standard deviation of $17.7$--$45.6$\,pp in five of the six cells (in the sixth, every seed ends at chance) indicates \emph{sampler collapse}, in which the sampler concentrates on a small set of high-score examples that the target then memorizes. This failure mode differs from magnitude amplification.
\begin{table}[h]
\centering
\caption{Magnitude-based online samplers on CIFAR-10/VGG-19-BN (all $N$ examples, 160 epochs, $\xi=0.1$; test accuracy \%, mean$\pm$std over 3 seeds). $Q$: share of visits to corrupted examples.}
\label{tab:magnitude_online}
\small
\begin{tabular}{lccc}
\toprule
Sampler & Clean & HN 10\% & HN 25\% \\
\midrule
Uniform SGD & $93.70\pm0.21$ & $86.59\pm0.22$ & $74.91\pm0.50$ \\
Gradient-norm IS ($Q=0.72$ / $0.83$) & $90.81\pm0.11$ & $10.00\pm0.00$ & $10.00\pm0.00$ \\
RHO-LOSS score, proxy reference & $20.21\pm17.68$ & $22.78\pm22.14$ & $10.00\pm0.00$ \\
RHO-LOSS score, holdout reference & $36.32\pm45.59$ & $59.99\pm43.31$ & $31.30\pm36.89$
 \\
\bottomrule
\end{tabular}
\end{table}

\section{Static Pruning: Additional Results}
\label{app:exp:sweeps}

\subsection{Paired-$t$ analysis of the headline CIFAR-10 cells}
\label{app:paired_t_headline}
Random and \drisShort{} were run with $n=7$ seeds at the two headline CIFAR-10 cells, and, seed by seed, their runs share the noise draw and the target initialization. Pairing the runs by seed removes this shared variability (\Cref{tab:paired_t_headline}): the gain under high-norm-$25\%$ noise is $+7.21$\,pp ($p<10^{-5}$), and the cost on clean data is $-0.44$\,pp ($p=0.0035$).
\begin{table}[ht!]
\centering
\caption{Paired-$t$ analysis for the headline CIFAR-10 / VGG-19 results ($\text{DR-IS-static} - \text{Random}$) across $n=7$ seeds paired by seed.}
\label{tab:paired_t_headline}
\small
\begin{tabular}{lccccc}
\toprule
Cell & $n$ & mean diff (pp) & paired SD (pp) & $t$ & $p$ (two-sided) \\
\midrule
Clean & $7$ & $-0.44$ & $0.25$ & $-4.66$ & $0.0035$ \\
High-norm 25\% & $7$ & $+7.21$ & $1.23$ & $+15.54$ & $4.5\cdot10^{-6}$
 \\
\bottomrule
\end{tabular}
\end{table}

\subsection{Proxy-seed variability}
\label{app:proxy_seeds}
In all other experiments, the proxies have fixed seeds (\Cref{app:seeding}), so the reported standard deviations do not include proxy variability. At the headline cell (CIFAR-10, high-norm-$25\%$, $\alpha=0.25$) we therefore ran TRV, the \drisShort{} score, and \drisShort{}-online with a freshly trained $K=3$ ensemble for each of three run seeds (proxy seeds $1000+10s+k$). TRV reaches $86.04\pm0.11\%$ with subset contamination $0.036\pm0.003$ (fixed proxies: $85.93\pm0.22\%$, $0.036$); the \drisShort{} score $81.50\pm0.81\%$ with $0.079\pm0.008$ ($81.86\pm0.82\%$, $0.072$); and \drisShort{}-online $84.35\pm0.19\%$ with a relative visit share of $0.51$ ($84.52\pm0.26\%$, $0.50$). Every difference is within one standard deviation, and the ordering of the rules does not change.

\subsection{Keep-fraction and corruption-rate sweeps}
\label{app:exp:alpha}
\label{app:exp:eps}
\Cref{tab:alpha_eps} reports the joint grid over keep fraction and corruption rate on CIFAR-10. TRV is the best rule in $6$ of the $9$ cells, including every cell at $\alpha=0.25$. The \drisShort{} score leads or ties AUM at $\varepsilon=0.1$ and at $\alpha=\varepsilon=0.25$. AUM is best at $(\alpha,\varepsilon)=(0.5,0.4)$ and at $\alpha=0.75$, $\varepsilon\ge0.25$. In these cells, the selection reaches into the low-score samples, where, by \Cref{prop:ranks}(e), rank statistics cannot separate the consistently hardest examples from the consistently easiest. Over the same grid, \drisShort{}'s subset contamination grows with $\alpha$ as the selection reaches into the zero-variance pool ($0.072\to0.167$ at $\varepsilon=0.25$), while AUM stays nearly clean except at $(\alpha,\varepsilon)=(0.75,0.40)$.
\begin{table}[!htb]
\centering
\caption{Regime map on CIFAR-10/VGG-19-BN under high-norm noise: test accuracy (\%, mean$\pm$std, 3 seeds; $n=7$ for Random and \drisShort{} at $\alpha=\varepsilon=0.25$) over keep-fraction $\alpha$ and corruption rate $\varepsilon$. TRV: TRV from the same proxies.}
\label{tab:alpha_eps}
\footnotesize
\setlength{\tabcolsep}{2pt}
\begin{tabular}{llccc}
\toprule
$\alpha$ & Method & $\varepsilon=0.10$ & $\varepsilon=0.25$ & $\varepsilon=0.40$ \\
\midrule
\multirow{4}{*}{0.25}
 & Random       & $82.81\pm0.15$ & $74.66\pm0.54$ & $65.33\pm0.45$ \\
 & AUM          & $79.73\pm1.58$ & $79.69\pm0.28$ & $76.83\pm0.37$ \\
 & \drisShort{} & $84.55\pm0.94$ & $81.86\pm0.82$ & $70.86\pm1.11$ \\
 & TRV         & $\mathbf{86.62\pm0.12}$ & $\mathbf{85.93\pm0.22}$ & $\mathbf{80.27\pm0.55}$ \\
\midrule
\multirow{4}{*}{0.50}
 & Random       & $85.58\pm0.24$ & $76.06\pm0.62$ & $65.39\pm0.16$ \\
 & AUM          & $85.10\pm1.21$ & $85.26\pm0.33$ & $\mathbf{84.14\pm0.10}$ \\
 & \drisShort{} & $\mathbf{88.58\pm0.29}$ & $83.55\pm0.23$ & $69.85\pm1.00$ \\
 & TRV         & $\mathbf{88.72\pm0.11}$ & $\mathbf{86.80\pm0.67}$ & $77.97\pm0.29$ \\
\midrule
\multirow{4}{*}{0.75}
 & Random       & $86.38\pm0.37$ & $75.70\pm0.16$ & $62.10\pm0.27$ \\
 & AUM          & $89.50\pm0.21$ & $\mathbf{87.22\pm0.21}$ & $\mathbf{77.68\pm0.53}$ \\
 & \drisShort{} & $89.57\pm0.29$ & $80.99\pm0.38$ & $65.88\pm1.10$ \\
 & TRV         & $\mathbf{90.63\pm0.27}$ & $84.16\pm0.18$ & $71.12\pm0.49$
 \\
\bottomrule
\end{tabular}
\end{table}

\Cref{fig:breakdown_curve} sweeps $\varepsilon$ at $\alpha=0.25$ up to the breakdown point $\varepsilon=\frac12$, which lies just outside \Cref{assu:contamination}. \drisShort{} degrades gracefully through $\varepsilon=0.25$ (cis $0.072$), decays through $0.45$ (cis $0.408$), and collapses at $0.50$ (cis $0.560$), $5.7$\,pp below Random. The escaping tail grows with $\varepsilon$ and loosens \eqref{eq:contam}. TRV holds up longer. It reaches $80.27\%$ at $\varepsilon=0.40$ (cis $0.194$) and $74.92\%$ at $0.45$ (cis $0.307$), and even at $\varepsilon=0.50$ it stays $7.5$\,pp above Random ($65.10$ vs.\ $57.62\%$, cis $0.470$). Its contamination is below \drisShort{}'s at every $\varepsilon$, as \Cref{prop:ranks}(d) predicts: the examples that escape the top ranks under a single proxy enter \drisShort{}'s subset but not TRV's. The proxies' training accuracy on $\Dcorr$ does not rise with $\varepsilon$ (right panel), so both failures come from a loss of separation, not from memorization. AUM is nearly flat ($76.28$--$79.73\%$, cis ${\le}0.001$); it overtakes \drisShort{} from $\varepsilon=0.40$ and TRV from $\varepsilon=0.45$.
\begin{figure}[h]
\centering
\includegraphics[width=\linewidth]{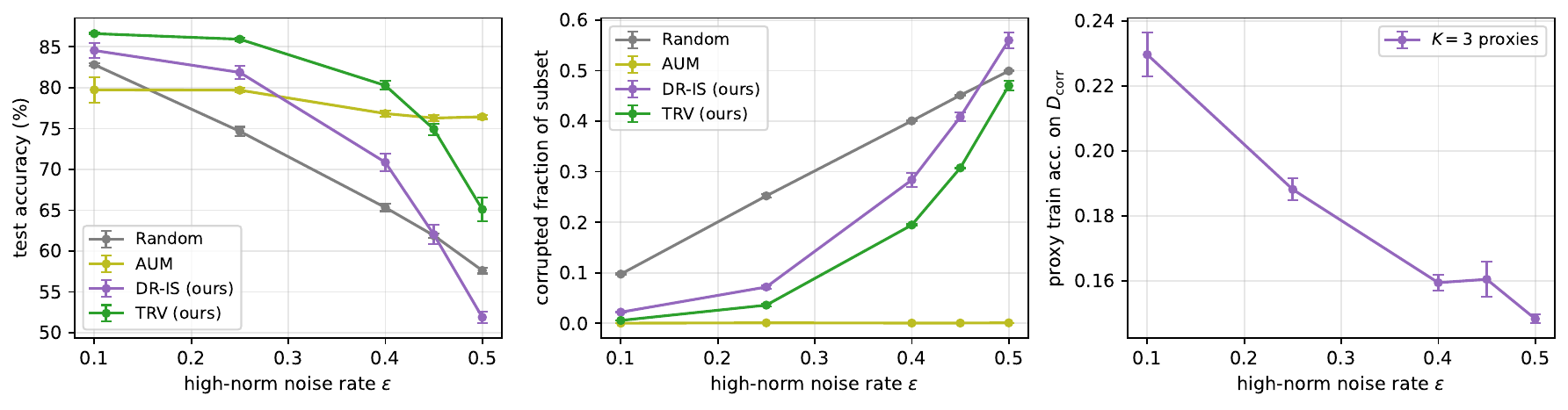}
\caption{Corruption-rate breakdown sweep on CIFAR-10/VGG-19-BN, $\alpha=0.25$, high-norm noise. TRV uses the same proxies as \drisShort{} and AUM. Left: test accuracy vs.\ $\varepsilon$. Middle: fraction of the selected subset that is corrupted. Right: training-set accuracy on $\Dcorr$ averaged over the $K=3$ proxies; a rise towards 1 would indicate that the proxies fit the corrupted labels.}
\label{fig:breakdown_curve}
\end{figure}

\subsection{Calibration against DivideMix}
\label{app:dividemix}
\Cref{tab:dividemix} compares the selection rules against full DivideMix \citep{li2020dividemix}, run at our batch size $128$. DivideMix's strength under this noise is consistent with \Cref{prop:suboptimality}: it \emph{discards the labels of} high-loss examples and trains on them semi-supervised, whereas the theorem penalizes scores that \emph{increase selection mass} with the loss. GMM-select ranks examples by $P(\text{clean}\mid\ell_i)$, which decreases monotonically in the loss and therefore essentially inverts loss-aligned noise; it is the small-loss end of \Cref{prop:suboptimality}. It yields the cleanest subset (cis $0.0005$, against $0.001$ for AUM and $0.072$ for \drisShort{}), but the accuracy ordering is reversed ($77.40<79.69<81.86$). Across these three selectors, subset cleanliness is inversely related to accuracy: a clean subset of easy examples is not enough. Full DivideMix takes $4.8\times$ the wall-clock time of online \drisShort{} ($1{,}402$\,s). Stopped at the static budget ($1{,}425$\,s), it reaches $83.68\%$ on all $N$ examples, an option that requires revisiting the full dataset every epoch.
\begin{table}[h]
\centering
\caption{Calibration against objective-based noisy-label learning, CIFAR-10, high-norm-$25\%$. Wall-clock is target training only (proxy training excluded). GMM-select ranks examples by DivideMix's averaged loss-GMM clean posterior over the same $K=3$ proxies and keeps the top $\alpha N$.}
\label{tab:dividemix}
\small
\resizebox{\linewidth}{!}{\begin{tabular}{lccccc}
\toprule
Method & Batch & Test acc. & Wall-clock (s) & cis & Data / epochs \\
\midrule
Random & 128 & $74.66\pm0.54$ & $1544\pm17$ & $0.252$ & $\alpha N$ / 640 \\
AUM & 128 & $79.69\pm0.28$ & $1522\pm44$ & $0.001$ & $\alpha N$ / 640 \\
\drisShort{} & 128 & $81.86\pm0.82$ & $1559\pm28$ & $0.072$ & $\alpha N$ / 640 \\
GMM-select & 128 & $77.40\pm0.99$ & $1923\pm32$ & $0.0005$ & $\alpha N$ / 640 \\
\midrule
Full DivideMix & 128 & $84.78\pm0.88$ & $6677\pm79$ & --- & $N$ / 160 \\
Full DivideMix, time-budgeted & 128 & $83.68\pm0.44$ & $1425\pm14$ & --- & $N$ / $40\pm1$
 \\
\bottomrule
\end{tabular}}
\end{table}

\subsection{TRV: number of proxies}
\label{app:exp:band}
Selecting the top $\alpha N$ by TRV keeps the $\alpha N$ examples whose mean rank is closest to $\frac12$ (TRV decreases in $|\overline\rho_i-\frac12|$), in the spirit of moderate coresets \citep{xia2023moderate}. It excludes the corrupted bulk by the same mechanism as \drisShort{}: when all $K$ proxies rank $i$ at least $1-\tau$, which for $i\in\Dbulk$ has probability at least $(1-\gamma)^K$ for independent proxies, then $\overline\rho_i\ge1-\tau$ and $\overline\rho_i(1-\overline\rho_i)\le\tau$. \Cref{prop:band} states the guarantees; here we collect further experiments.

\emph{How many proxies.} Scoring the cached proxies without retraining the target, TRV's cis at high-norm-$25\%$, averaged over all $K$-subsets of the three proxies and over the cached seeds 0--6, falls from $0.059$ ($K=1$) to $0.042$ ($K=2$) and $0.034$ ($K=3$), below \drisShort{}'s $0.087$ and $0.071$ at $K=2,3$. In the $K\le32$ sweep at high-norm-$10\%$ it falls from $0.032$ to $0.003$, against $0.033$ to $0.013$ for \drisShort{}. The selections change with $K$: the $K=1$ TRV shares only $54\%$ of its subset with the $K=3$ TRV. Training the target confirms that a single proxy is not enough under noise: TRV from the ResNet-20 proxy alone reaches $83.10\pm0.90\%$ (cis $0.085$) against $85.93\pm0.22\%$ at $K=3$, although on clean data it is the better of the two ($87.21$ vs.\ $86.42\%$).

\begin{table}[t]
\centering
\caption{Rank variance vs.\ variance of raw or z-scored per-proxy losses (CIFAR-10, high-norm-$25\%$, $\alpha=0.25$). $K=3$: deployed ensemble; $K=8$: one run of an 8-proxy ensemble (\Cref{app:exp:capacity}). cis: corrupted fraction of the subset (chance $0.25$). $M$: fraction of clean examples scoring above the 95th percentile of the corrupted ones.}
\label{tab:score_ablation}
\footnotesize
\setlength{\tabcolsep}{3pt}
\begin{tabular}{lccc}
\toprule
Score & cis ($K{=}3$) & $M$ ($K{=}3$) & cis ($K{=}8$) \\
\midrule
Rank var.\ (\drisShort{}) & $0.07$ & $0.226$ & $0.09$ \\
Loss variance                & $0.53$ & $0.008$ & $0.55$ \\
z-scored loss variance       & $0.54$ & $0.008$ & $0.54$ \\
TRV $\overline\rho(1-\overline\rho)$ & $0.04$ & $0.400$ & $0.02$
 \\
\bottomrule
\end{tabular}
\end{table}

\section{Premises of the Analysis and Label-Free Diagnostics}
\label{app:premises}

\subsection{Score ablation and scoring epoch}
\label{app:exp:capacity}
The score ablation (\Cref{tab:score_ablation}) shows that the rank transform carries the method. With raw or z-scored losses, the variance of the corrupted tail's magnitudes dominates. For example, losses of $8$, $10$ and $12$ across three proxies give a large loss variance, yet the example is consistently the hardest and has near-zero rank variance. Z-scoring the losses does not fix the heavy tail. Boundedness also matters for the theory: it is what permits McDiarmid's inequality in \Cref{thm:main}, and no distribution-free certificate exists for an unbounded loss variance. Ranks do fluctuate when many examples have similar losses. This local volatility is the boundary signal of \Cref{assu:bdry_disagreement}, and it does not contaminate the selection, because bulk-corrupted examples sit in the extreme upper tail, where there is nothing above them to swap with.

\emph{Scoring epoch} (\Cref{tab:scoring_epoch}). Under the native $40$-epoch schedule, epochs $10$--$20$ are optimal and the default epoch $20$ is near-optimal. The degradation through epoch $40$ is graceful (${\approx}12\%$ corrupted) and sharpens with $K$ (epoch $30$: $12.4\%$ at $K=3$, $20.9\%$ at $K=8$), since a tighter variance estimator also tracks the memorization-driven signal more closely. An extended sweep over a $100$-epoch schedule places the onset of memorization at about epoch $50$, after which the $K=8$ selection becomes corruption-enriched ($31$--$33\%$ in the extended sweep; \Cref{tab:scoring_epoch} shows the native schedule). ``Larger $K$ is better'' therefore holds only before memorization, which check (C3) (\Cref{app:diagnostics}) verifies without labels.
\begin{table}[h]
\centering
\caption{Corruption in the selected subset (\%, chance $25\%$) vs.\ proxy scoring epoch, CIFAR-10, high-norm-$25\%$, $\alpha=0.25$, single run. $K$ is the first $K$ of $8$ proxies that cycle through ResNet-20, MobileNetV2$_{0.5\times}$, ShuffleNetV2$_{0.5\times}$, and RepVGG-A0 ($K=3$ has the deployed architectures but different seeds). The cosine schedule is set for $40$ epochs; training continues to $100$, where the uncapped cosine raises the learning rate again (peak at epoch $80$), so the columns right of the bar are a warm restart rather than continued training.}
\label{tab:scoring_epoch}
\small
\begin{tabular}{lcccc|cccc}
\toprule
Scoring epoch & 10 & 20 & 30 & 40 & 50 & 60 & 80 & 100 \\
\midrule
$K=3$ (default) & 4.8 & 8.3 & 12.4 & 11.9 & 12.7 & 10.4 & 9.3 & 15.6 \\
$K=4$           & 4.4 & 11.8 & 19.0 & 17.9 & 17.6 & 20.9 & 21.5 & 27.2 \\
$K=8$           & 3.4 & 10.0 & 20.9 & 18.3 & 19.0 & 24.5 & 20.3 & 31.7
 \\
\bottomrule
\end{tabular}
\end{table}

\subsection{Proxy diversity}
\label{app:exp:proxy_diversity}
We compare three $K=3$ configurations on CIFAR-10, $\alpha=0.25$, high-norm-$25\%$: K1 is three ResNet-20s, K2 is three MobileNetV2$_{0.5\times}$s, and K3 is the heterogeneous default. The per-proxy training accuracies are within $2.3$\,pp across K1 and K2. The boundary variance $\widehat\tau^2_{\mathrm{bdry}}$ is flat across configurations ($0.0085$--$0.0145$), and the homogeneous K1 actually has the largest. All three configurations work (cis $\in[0.064,0.077]$, accuracy $\in[80.90,83.13]$). The maximum variance among corrupted examples ($\widehat\tau^2_{\mathrm{corr,max}}\approx0.12$--$0.15$) exceeds the boundary median. The theory allows this, because the maximum occurs in the escaping tail, while the bulk concentrates sharply. A homogeneous ensemble is therefore sufficient; diversity is a recommended default, not a precondition. Nor do the proxies need to be deep networks trained from scratch: on Food-101 they are linear classifiers on frozen ImageNet-pretrained features (\Cref{app:proxy_details}), and the rank scores still separate the injected flips there (\Cref{app:food101}).
\begin{table}[h]
\centering
\caption{Proxy-diversity ablation on CIFAR-10/VGG-19-BN/$\alpha=0.25$. K1 = $K=3$ ResNet-20s (homogeneous). K2 = $K=3$ MobileNetV2$_{0.5\times}$s (homogeneous, different architecture). K3 = $K=3$ heterogeneous proxies (paper default). All cells under 25\% high-norm noise; 2 new seeds for K1 and K2; for K3, 7 seeds for accuracy and cis (as in \Cref{tab:cifar_pruning}) and the same seven seeds for $\widehat\tau^2$ and proxy train-acc. We expected heterogeneous priors to widen $\widehat\tau^2_{\mathrm{bdry}}$ (Assumption~3); they do not: the homogeneous K1 has the largest value.}
\label{tab:proxy_diversity}
\small
\begin{tabular}{lccccc}
\toprule
Config & DR-IS acc & DR-IS corr-in-subset & $\widehat\tau^2_{\mathrm{corr}}$ & $\widehat\tau^2_{\mathrm{bdry}}$ & proxy train-acc \\
\midrule
K1 & 80.90\,$\pm$\,0.47 & 0.077\,$\pm$\,0.010 & 0.1451\,$\pm$\,0.0005 & 0.0145\,$\pm$\,0.0005 & 75.7\,$\pm$\,0.1 \\
K2 & 83.13\,$\pm$\,0.35 & 0.064\,$\pm$\,0.002 & 0.1350\,$\pm$\,0.0285 & 0.0085\,$\pm$\,0.0005 & 73.4\,$\pm$\,0.2 \\
K3 & 81.86\,$\pm$\,0.82 & 0.072\,$\pm$\,0.004 & 0.1376\,$\pm$\,0.0179 & 0.0091\,$\pm$\,0.0010 & 73.8\,$\pm$\,0.4
 \\
\bottomrule
\end{tabular}
\end{table}

\begin{figure}[h]
\centering
\includegraphics[width=0.55\linewidth]{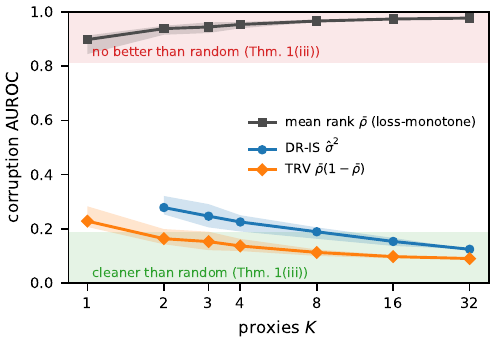}
\caption{Corruption AUROC against ensemble size $K$ (CIFAR-10, high-norm-$10\%$; mean and 10--90\% range over 20 random $K$-subsets of a separately trained 32-proxy ensemble, one subset at $K=32$; $K=1$ shown for the mean rank and TRV). Shaded: the regions of \Cref{prop:suboptimality}(iii) at $\alpha=0.25$. Averaging more proxies makes the loss consensus worse and both rank statistics better.}
\label{fig:auroc_K}
\end{figure}

\subsection{Details of \Cref{tab:auroc}}
\label{app:table1}
\Cref{tab:auroc} uses the cached per-example losses of the deployed $K=3$ proxies at their scoring epoch $20$, for CIFAR-10 with high-norm-$25\%$ noise and seeds 0--2, together with the corruption mask of \Cref{app:exp_details}. The five scores are the loss of the ResNet-20 proxy (\emph{loss, 1 proxy}), the mean loss over the three proxies, the mean rank $\overline\rho$, the \drisShort{} score $\rankvar$, and TRV $\widehat b$. For each score, the AUROC is the fraction of (corrupted, clean) pairs in which the corrupted example scores higher, with ties counted as $\frac12$, which is the expectation of the random tie-break of \Cref{sec:setting}; $x$ is the corrupted examples' mean score divided by the dataset mean. The \emph{allowed} column is the contamination interval of \Cref{prop:suboptimality}(ii) at the measured AUROC, with $\varepsilon=\alpha=0.25$, in the closed form of \Cref{app:prop_suboptimality}. The \emph{cis} column is the contamination of the top $\alpha N$ examples of the cached score; it differs from \Cref{tab:cifar_pruning} in the third digit, because that table averages separately trained runs. The last column applies \Cref{prop:suboptimality}(iii): the verdict $\mathrm{cis}\ge\varepsilon$ holds for every keep fraction with $\alpha(1-\alpha)\ge1-A_s$, and $\mathrm{cis}\le\varepsilon$ for every keep fraction with $\alpha(1-\alpha)\ge A_s$. At $\alpha=0.25$, a score with AUROC strictly between $0.1875$ and $0.8125$ gets no verdict, as is the case for \drisShort{} ($0.254$).

The \emph{est.} columns repeat the computation of the AUROC and $x$ with the pseudo-mask of check (C2) in place of the corruption mask: the examples whose given label gets probability below $\frac12$ from every proxy. It flags $23\%$ of the data (true $\varepsilon=0.25$), with precision $0.88$ and recall $0.80$. The estimates keep every ordering of the true columns, but they are optimistic for the two rank statistics, which are near zero on exactly the examples the pseudo-mask flags; for \drisShort{} the estimated AUROC ($0.167$) falls below the threshold $0.1875$, while the true one ($0.254$) does not. The estimates therefore support the orderings, not the absolute values.

\section{Pre-Deployment Diagnostics}
\label{app:diagnostics}
We state each check with its scope. Check (C2), the label-free estimate of $\varepsilon$, of the AUROC and of $x$, is described in \Cref{sec:checks} and evaluated in \Cref{tab:auroc}.

\paragraph{(C1) Disagreement histogram (unsupervised; for \Cref{assu:bdry_disagreement}).} Plot the histogram of $\rankvar$ over the full dataset. A healthy ensemble shows a mode at zero, formed mostly by the corrupted bulk and by examples that every proxy finds easy, and a long right tail of clean boundary examples. If the histogram is unimodal and collapsed at zero, the proxies are too similar or have memorized the labels, and the method should not be deployed. \Cref{fig:K_ablation} shows the healthy case, separately for clean and corrupted examples.

\paragraph{(C3) Scoring-epoch check (unsupervised; for \Cref{assu:adv_concentration} over training time).} Compute $\rankvar$ at several proxy epochs and form the cross-epoch Spearman correlation matrix of these scores. After the onset of memorization, the late-epoch scores form a block that is correlated within itself and nearly uncorrelated with earlier epochs, because the variance becomes memorization-driven; the scores should be taken before that block. In a single proxy run on CIFAR-10 at high-norm-$25\%$ with $K=8$ and a $100$-epoch cosine schedule, the corrupted examples' scores correlate at $0.90$ between epochs $80$ and $100$ but at most $0.10$ with epochs $10$--$50$, and this late block coincides with the epochs at which the $K=8$ selection becomes corruption-enriched. At $K=3$ the corrupted examples' correlations before onset are only $0.13$--$0.32$, so the check is mostly informative for larger ensembles. The check should be applied within the proxies' learning-rate schedule: past its end, as in the right part of \Cref{tab:scoring_epoch}, the scores also follow the restarted learning rate, so the block structure no longer isolates memorization.

\section{Food-101: Full Results}
\label{app:food101}
\emph{Setup.} Data, proxies, and targets are described in \Cref{app:exp_details}. E1 uses the natural label noise of Food-101 only; E2 flips an additional $25\%$ of the labels by gradient norm. On E1, the subsets do not depend on the seed (fixed linear-probe proxies, natural noise only); the three seeds differ in the target initialization and data order (\Cref{app:seeding}).

\emph{Mechanism diagnostic.} The natural noise comes with no ground-truth mask. To describe what each rule keeps, a ResNet-50 trained on the hand-cleaned test set flags the training examples whose label disagrees with its prediction. The flags ($25\%$ of the data) mostly capture model imprecision, so they serve only to describe the composition of the subsets, not as a measure of noise.

\emph{E1.} \drisShort{}'s subset contains flagged examples at the background rate ($0.25$), so no filtering occurs, and \drisShort{} matches Random. \emph{E2.} EL2N collapses ($35.51\%$, $49\%$ of its subset carrying an injected flip, about twice chance). \drisShort{} reaches $73.67\%$ with $18\%$ of its subset corrupt. The margin is smaller than on CIFAR, which is consistent with the weaker hypothesis diversity of linear probes on a $101$-class task. AUM filters both the injected and much of the natural ``noise'', which helps on E2 and costs $4.9$\,pp on E1. TRV halves \drisShort{}'s injected contamination on E2 ($0.09$) and reaches $75.35\%$, but falls $1.3$\,pp below Random on E1.
\begin{table}[h]
\centering
\caption{Food-101 / ResNet-50, $\alpha=0.5$, 3 seeds. \textbf{E1}: natural label noise only; \textbf{E2}: an additional $25\%$ of high-norm flips. Test accuracy (\%, mean$\pm$std). E1 Est.\ Noise: estimated candidate-noise fraction of the subset (background rate $0.25$); E2 Frac.\ Corrupt: fraction of the subset carrying an injected flip (chance $0.25$).}
\label{tab:food101}
\small
\begin{tabular}{lcccc}
\toprule
& \multicolumn{2}{c}{E1 (Natural Noise)} & \multicolumn{2}{c}{E2 (+25\% High-norm)} \\
\cmidrule(lr){2-3} \cmidrule(lr){4-5}
Method & Test Acc. & Est. Noise & Test Acc. & Frac. Corrupt \\
\midrule
Random               & \underline{$82.67 \pm 0.25$} & $0.25 \pm 0.00$ & $70.69 \pm 0.24$ & $0.25 \pm 0.00$ \\
EL2N (ensemble)      & $74.23 \pm 0.29$ & $0.44$ & $35.51 \pm 0.12$ & $0.49$ \\
AUM (top-$\alpha$)   & $77.79 \pm 0.15$ & $\mathbf{0.07}$ & $\mathbf{75.56 \pm 0.03}$ & $\mathbf{0.01}$ \\
\drisShort{} (ours)  & $\mathbf{82.88 \pm 0.12}$ & $0.25$ & $73.67 \pm 0.32$ & $0.18$ \\
TRV $\overline\rho(1-\overline\rho)$ & $81.39 \pm 0.15$ & \underline{$0.20$} & \underline{$75.35 \pm 0.28$} & \underline{$0.09$}
 \\
\bottomrule
\end{tabular}
\end{table}

\section{Robustness to the Corruption Model}
\label{app:noise_models}
The main experiments use two corruption models: uniform flips, and high-norm flips of the examples with the largest gradient norm under a ResNet-20 noise model. To check that the results do not depend on this choice, we vary the corruption model in two ways (construction in \Cref{app:exp_details}). First, the high-norm flips are chosen by a different noise model, a ResNet-18 or a VGG-16-BN; their top-$10\%$ gradient-norm sets overlap the default one by only $26$--$28\%$. Second, \emph{volatile-margin} flips corrupt the $25\%$ of examples whose margin fluctuates most during the noise model's training. These are mostly low-norm examples, and their set overlaps the high-norm one by only $12\%$. Everything else is as in the main experiments: $\varepsilon=0.25$, the same $K=3$ proxies scored at epoch $20$, $\alpha=0.25$ with step parity for static pruning, $\xi=0.1$ for online sampling, and $3$ seeds. The default column of \Cref{tab:noise_static,tab:noise_online} repeats \Cref{tab:cifar_pruning,tab:online_is}.

\emph{Static pruning} (\Cref{tab:noise_static}). TRV is the best rule and \drisShort{} the second best under every corruption model on both datasets. Relative to Random, TRV gains $11.3$--$12.8$\,pp on CIFAR-10 and $10.0$--$10.7$\,pp on CIFAR-100, and \drisShort{} gains $7.2$--$10.1$\,pp and $9.4$--$9.7$\,pp. The subset contamination remains in the same range as with the default noise model: $0.030$--$0.039$ for TRV and $0.066$--$0.079$ for \drisShort{} on CIFAR-10, and $0.025$--$0.038$ and $0.055$--$0.084$ on CIFAR-100. Under volatile-margin noise EL2N collapses to $6.58\pm1.52\%$ with an $87\%$-corrupted subset, as it does under high-norm noise ($9.39\%$, $87\%$), so this corruption is also loss-aligned. AUM keeps an almost clean subset under every corruption model (cis ${\le}0.001$), but under volatile-margin noise it gains only $1.7$\,pp over Random ($72.67$ vs.\ $71.01\%$): as for GMM-select (\Cref{app:dividemix}), a clean subset of easy examples is not enough.

\emph{Online sampling} (\Cref{tab:noise_online}). \drisShort{}-online improves on uniform SGD by $9.9$\,pp (ResNet-18 noise model), $9.7$\,pp (VGG-16-BN, over the two converged seeds) and $11.2$\,pp (volatile-margin), against $9.6$\,pp with the default noise model. Its relative visit share is $0.46$--$0.53$ in every setting, so the corrupted examples are visited about half as often as under uniform sampling, whichever examples are corrupted. TRV-online gains only $0.7$--$2.0$\,pp, with relative visit shares of $0.71$--$0.73$, consistent with its larger $x$ (\Cref{prop:ranks}). One of the three \drisShort{}-online seeds with the VGG-16-BN noise model collapsed to chance in the first epoch; \Cref{app:stability} analyzes this instability and shows that a learning-rate warm-up removes it.
\begin{table}[h]
\centering
\caption{Static pruning under four corruption models at $25\%$: high-norm flips chosen by the default ResNet-20 noise model (as in \Cref{tab:cifar_pruning}), by a ResNet-18 or by a VGG-16-BN noise model, and volatile-margin flips (CIFAR-10 only). $\alpha=0.25$ with step parity and the same $K=3$ proxies throughout.}
\label{tab:noise_static}
\small
\resizebox{\linewidth}{!}{\begin{tabular}{l cccc c ccc}
\toprule
& \multicolumn{4}{c}{\textsc{CIFAR-10 / VGG-19-BN}} && \multicolumn{3}{c}{\textsc{CIFAR-100 / ResNet-18}} \\
\cmidrule(lr){2-5}\cmidrule(lr){7-9}
Method & HN, ResNet-20 & HN, ResNet-18 & HN, VGG-16-BN & Volatile && HN, ResNet-20 & HN, ResNet-18 & HN, VGG-16-BN \\
\midrule
Random               & $74.66\pm0.54$ & $73.37\pm0.48$ & $74.90\pm0.27$ & $71.01\pm0.74$ && $51.53\pm0.59$ & $47.76\pm0.76$ & $48.61\pm0.38$ \\
AUM (top-$\alpha$)   & $79.69\pm0.28$ & $76.50\pm0.20$ & $77.70\pm0.50$ & $72.67\pm0.26$ && $53.32\pm0.14$ & $52.14\pm0.16$ & $52.83\pm0.22$ \\
\drisShort{} $\rankvar$ & \underline{$81.86\pm0.82$} & \underline{$81.28\pm0.29$} & \underline{$82.65\pm0.51$} & \underline{$81.14\pm1.13$} && \underline{$60.96\pm0.68$} & \underline{$57.47\pm0.07$} & \underline{$58.14\pm0.39$} \\
TRV $\overline\rho(1-\overline\rho)$ & $\mathbf{85.93\pm0.22}$ & $\mathbf{85.93\pm0.53}$ & $\mathbf{86.65\pm0.28}$ & $\mathbf{83.80\pm1.18}$ && $\mathbf{61.55\pm0.41}$ & $\mathbf{58.44\pm0.34}$ & $\mathbf{58.95\pm0.50}$
 \\
\bottomrule
\end{tabular}}
\end{table}

\begin{table}[h]
\centering
\caption{Online importance sampling under the four corruption models of \Cref{tab:noise_static} (CIFAR-10/VGG-19-BN, all $N$ examples, $160$ epochs, $\xi=0.1$, no warm-up). Test accuracy (\%, mean$\pm$std over 3 seeds) and $Q/\varepsilon$, the visit share of the corrupted examples relative to uniform sampling. $^\dagger$One of the three seeds collapsed to chance in the first epoch; the mean is over the two converged seeds. With a learning-rate warm-up all three converge, at $83.86\pm0.10$ (\Cref{tab:stability}).}
\label{tab:noise_online}
\small
\resizebox{\linewidth}{!}{\begin{tabular}{l cccc c cccc}
\toprule
& \multicolumn{4}{c}{Test accuracy} && \multicolumn{4}{c}{$Q/\varepsilon$} \\
\cmidrule(lr){2-5}\cmidrule(lr){7-10}
Method & HN, ResNet-20 & HN, ResNet-18 & HN, VGG-16-BN & Volatile && HN, ResNet-20 & HN, ResNet-18 & HN, VGG-16-BN & Volatile \\
\midrule
Uniform SGD     & $74.91\pm0.50$ & $73.51\pm0.23$ & $74.41\pm0.36$ & $76.20\pm0.19$ && $1.00$ & $1.00$ & $1.00$ & $1.00$ \\
TRV-online      & $76.02\pm0.64$ & $74.19\pm0.60$ & $75.27\pm0.78$ & $78.21\pm1.03$ && $.72$ & $.73$ & $.72$ & $.71$ \\
\drisShort{}-online & $\mathbf{84.52\pm0.26}$ & $\mathbf{83.40\pm0.60}$ & $\mathbf{84.09\pm0.59}$$^\dagger$ & $\mathbf{87.37\pm0.16}$ && $.49$ & $.53$ & $.49$ & $.46$
 \\
\bottomrule
\end{tabular}}
\end{table}

\section{Stability of Online Sampling and Learning-Rate Warm-Up}
\label{app:stability}
\emph{Why online \drisShort{} can be unstable.} \Cref{prop:online} bounds the price of suppressing visits. The sampler is unbiased, but its step has second moment $\kappa_{\mathrm{eff}}\,\E\|g_U\|^2$, and a step with learning rate $\eta$ decreases the empirical risk in expectation only if $\eta<2\|\nabla\widehat R\|^2/(L\kappa_{\mathrm{eff}}\E\|g_U\|^2)$. Compared with uniform SGD, the largest admissible learning rate shrinks by a factor of $\kappa_{\mathrm{eff}}$. Early in training, before the gradient norms are related to the sampling weights, $\kappa_{\mathrm{eff}}$ is close to the average weight $\bar\omega=\frac1N\sum_i\omega_i$ (\Cref{app:prop_online}). On the high-norm-$25\%$ case, $\bar\omega=2.8$ for \drisShort{} and $1.2$ for TRV at $\xi=0.1$, and $4.5$ and $1.3$ at $\xi=0.03$. The corrupted examples carry the largest weights: their average weight is $4.5$ under \drisShort{} and $1.7$ under TRV at $\xi=0.1$, and $7.9$ under \drisShort{} at $\xi=0.03$. Starting at learning rate $0.1$, online \drisShort{} therefore takes steps whose second moment is close to three times that of uniform SGD at the same learning rate, in the first epochs, when the randomly initialized network is least stable. The argument predicts that instability, if it occurs, appears at the start of training, affects \drisShort{}-online rather than TRV-online or uniform SGD, becomes more likely as $\xi$ shrinks, and disappears when the learning rate starts small. A learning-rate warm-up does exactly this without changing the sampler or its unbiasedness.

\begin{table}[!ht]
\centering
\caption{Online training without and with a $5$-epoch linear learning-rate warm-up (WU), applied to both methods (CIFAR-10/VGG-19-BN, all $N$ examples, $160$ epochs, $\xi=0.1$). Without warm-up the learning rate is $0.1$ from the first step. Test accuracy (\%, mean$\pm$std over 3 seeds). \emph{Collapsed}: \drisShort{}-online seeds that collapsed to chance in the first epoch without warm-up; the accuracy next to it is over the converged seeds. No uniform-SGD or TRV-online run collapsed.}
\label{tab:stability}
\small
\resizebox{\linewidth}{!}{\begin{tabular}{l ccc ccc}
\toprule
Setting & Uniform & Uniform, WU & TRV-online & \drisShort{}-online & Collapsed & \drisShort{}-online, WU \\
\midrule
HN 25\%, ResNet-20 (default) & $74.91\pm0.50$ & $74.58\pm0.41$ & $76.02\pm0.64$ & $84.52\pm0.26$ & $0/3$ & $84.27\pm0.22$ \\
HN 25\%, VGG-16-BN           & $74.41\pm0.36$ & $74.27\pm0.69$ & $75.27\pm0.78$ & $84.09\pm0.59$ & $1/3$ & $83.86\pm0.10$ \\
HN 40\%, ResNet-20           & $61.52\pm0.37$ & $61.41\pm0.13$ & $62.05\pm0.40$ & $76.75\pm0.18$ & $1/3$ & $76.30\pm0.29$
 \\
\bottomrule
\end{tabular}}
\end{table}

\emph{Observed instability.} Without warm-up, three \drisShort{}-online runs collapsed: one of three seeds with the VGG-16-BN noise model, one of three at high-norm-$40\%$, and one of three at $\xi=0.03$ (\Cref{tab:visits_fit}); every other \drisShort{}-online run at $\xi\ge0.1$ converged, including all runs of \Cref{tab:online_is}. Each collapse occurred in the first epoch: the test accuracy never exceeded $10.2\%$, and from the second epoch on the training loss stayed at about $\log 10\approx2.30$, the loss of a uniform prediction over the $10$ classes. None of the uniform-SGD or TRV-online runs collapsed. At $\xi=0.03$ a second seed degraded to $59.5\%$ without collapsing.

\emph{Smoothing.} The smoothing sweep at high-norm-$25\%$ of \Cref{app:online_ablations} (\Cref{tab:visits_fit}, without warm-up) shows the same trade-off from the side of the sampler. As $\xi$ decreases from $1$ to $0.3$, $0.1$ and $0.03$, the weight cap $1+1/\xi$ grows from $2$ to $4.3$, $11$ and $34$, and the average weight $\bar\omega$ of online \drisShort{} from $1.3$ to $1.8$, $2.8$ and $4.5$. Suppression grows with it: the relative visit share of the corrupted examples falls from $0.72$ to $0.57$, $0.50$ and $0.46$. Down to $\xi=0.1$, every seed converges and the accuracy rises with suppression ($76.1$, $80.9$ and $84.5\%$). At $\xi=0.03$ one seed collapses and a second degrades to $59.5\%$, so the extra suppression no longer pays off. $\xi=0.1$, fixed before any of these runs, is the smallest value that was stable without warm-up.

\emph{Warm-up} (\Cref{tab:stability}). With a $5$-epoch linear warm-up, in which the learning rate rises from $0.02$ in the first epoch to $0.1$ in the fifth, followed by the same cosine schedule and applied to both methods, every run converged. \drisShort{}-online then reaches $84.27\pm0.22\%$, $83.86\pm0.10\%$ and $76.30\pm0.29\%$ in the three settings, against $74.58\pm0.41\%$, $74.27\pm0.69\%$ and $61.41\pm0.13\%$ for uniform SGD, a gain of $9.7$, $9.6$ and $14.9$\,pp. The warm-up costs nothing: it changes the accuracy of the converged runs by at most $0.45$\,pp for either method. At $\varepsilon=0.40$, the gain is the largest of all online settings, although the sampler suppresses less: the corrupted examples receive $0.82$ of their uniform visit share under \drisShort{}-online ($0.49$ at $\varepsilon=0.25$) and $0.91$ under TRV-online, which gains only $0.5$\,pp. We therefore recommend the warm-up as the default for online \drisShort{}; all online results without it are reported as run.

\section{Real Human Label Noise: CIFAR-10N}
\label{app:cifar10n}
CIFAR-10N \citep{wei2022learning} relabels the CIFAR-10 training set with three human annotations per image, so its corruption mask is known: an example is corrupted if its human label differs from the original one. We use two of its label sets: \emph{aggre}, the majority vote ($\varepsilon=0.090$), and \emph{worst}, which takes a wrong annotation whenever one exists ($\varepsilon=0.402$). Everything else is as in the main CIFAR-10 experiments: the same $K=3$ proxies scored at epoch $20$, VGG-19-BN, $\alpha=0.25$ with step parity for static pruning, and $\xi=0.1$ for online sampling, with the $5$-epoch learning-rate warm-up of \Cref{app:stability} for both online methods.

\emph{Pre-registered predictions.} Before any target run, we trained the proxies, measured the AUROC and $x$ of every score against the human mask, and recorded the predictions (\Cref{tab:cifar10n_theory}). Both label sets are loss-aligned by \Cref{def:loss_aligned}: the single-proxy loss has AUROC $0.82$ and $0.85$. \Cref{prop:suboptimality}(iii) then places every loss-based score, EL2N included, at or above random contamination at $\alpha=0.25$, and gives no verdict for the two rank statistics. The recorded hypotheses were EL2N $<$ Random $<$ TRV for static pruning and \drisShort{}-online $>$ uniform SGD for online sampling, on both label sets.

\emph{Outcome.} All four hypotheses hold (\Cref{tab:cifar10n_static,tab:cifar10n_online}). EL2N keeps subsets that are $31\%$ and $92\%$ corrupted and falls to $48.0\%$ and $5.4\%$; under \emph{worst} its subset is dominated by consistent human error between classes, which the target learns, so it ends below chance. TRV improves on Random by $3.2$\,pp under \emph{aggre} and $4.5$\,pp under \emph{worst}, and \drisShort{} by $0.9$ and $9.1$\,pp. Online, \drisShort{}-online visits the corrupted examples $0.59$ and $0.82$ times as often as uniform sampling, fits $45\%$ and $33\%$ of their wrong labels against $98\%$ and $86\%$ for uniform SGD, and improves the final accuracy by $0.5$ and $12.4$\,pp; under \emph{worst} it ends within $0.9$\,pp of uniform SGD's test-selected early-stopping bound.

\emph{Departures from the synthetic-noise results.} Under \emph{worst}, \drisShort{} is the better pruning rule: it reaches $71.4\%$ against $66.8\%$ for TRV and keeps the cleaner subset ($0.298$ vs.\ $0.335$), in line with its lower AUROC ($0.379$ vs.\ $0.407$). The measured AUROC therefore still identifies the better threshold, but the order seen under every synthetic corruption, TRV below \drisShort{}, reverses at $40\%$ human noise. AUM keeps an almost clean subset (cis $0.014$) and is the best static rule under \emph{worst} ($76.9\%$); under high-norm noise at the same $\alpha$ and $\varepsilon=0.40$, TRV was ahead of it (\Cref{tab:alpha_eps}). Under \emph{aggre}, AUM keeps mostly easy examples and falls $5.2$\,pp below Random.

\begin{table}[h]
\centering
\caption{CIFAR-10N, pre-registered: corruption-detection AUROC and score ratio $x$ of each score against the human-noise mask (deployed $K=3$), and the verdict of \Cref{prop:suboptimality}(iii) at $\alpha=0.25$ ($\ge\varepsilon$: no cleaner than random selection; ---: no verdict). Computed from the proxies and recorded before any target run.}
\label{tab:cifar10n_theory}
\small
\begin{tabular}{l ccc c ccc}
\toprule
& \multicolumn{3}{c}{aggre ($\varepsilon=0.090$)} && \multicolumn{3}{c}{worst ($\varepsilon=0.402$)} \\
\cmidrule(lr){2-4}\cmidrule(lr){6-8}
Score & AUROC & $x$ & (iii) && AUROC & $x$ & (iii) \\
\midrule
loss, 1 proxy        & $.823$ & $3.68$ & $\ge\varepsilon$ && $.849$ & $1.78$ & $\ge\varepsilon$ \\
loss, mean           & $.918$ & $4.08$ & $\ge\varepsilon$ && $.919$ & $1.71$ & $\ge\varepsilon$ \\
EL2N (ens.)          & $.921$ & $2.88$ & $\ge\varepsilon$ && $.923$ & $1.43$ & $\ge\varepsilon$ \\
mean rank $\overline\rho$ & $.922$ & $1.67$ & $\ge\varepsilon$ && $.922$ & $1.45$ & $\ge\varepsilon$ \\
\drisShort{} $\rankvar$ & $.301$ & $0.54$ & --- && $.379$ & $0.80$ & --- \\
TRV $\widehat b$       & $.194$ & $0.62$ & --- && $.407$ & $0.92$ & ---
 \\
\bottomrule
\end{tabular}
\end{table}

\begin{table}[h]
\centering
\caption{Static pruning on CIFAR-10N (VGG-19-BN, $\alpha=0.25$, step parity, $K=3$ proxies as in \Cref{tab:cifar_pruning}): test accuracy (\%, mean$\pm$std over 3 seeds) and the corrupted fraction of the selected subset (cis; random $=\varepsilon$).}
\label{tab:cifar10n_static}
\small
\begin{tabular}{l cc c cc}
\toprule
& \multicolumn{2}{c}{aggre ($\varepsilon=0.090$)} && \multicolumn{2}{c}{worst ($\varepsilon=0.402$)} \\
\cmidrule(lr){2-3}\cmidrule(lr){5-6}
Method & Test acc. & cis && Test acc. & cis \\
\midrule
Random               & $83.38\pm0.32$ & $.091$ && $62.30\pm0.59$ & $.404$ \\
EL2N (ens.)          & $47.99\pm0.73$ & $.314$ && $\phantom{0}5.36\pm0.51$ & $.919$ \\
AUM (top-$\alpha$)   & $78.18\pm0.21$ & $.000$ && $\mathbf{76.92\pm0.50}$ & $.014$ \\
\drisShort{} $\rankvar$ & \underline{$84.32\pm0.18$} & $.042$ && \underline{$71.37\pm0.04$} & $.298$ \\
TRV $\overline\rho(1-\overline\rho)$ & $\mathbf{86.54\pm0.08}$ & $.017$ && $66.79\pm0.45$ & $.335$
 \\
\bottomrule
\end{tabular}
\end{table}

\begin{table}[h]
\centering
\caption{Online importance sampling on CIFAR-10N (VGG-19-BN, all $N$ examples, $160$ epochs, $\xi=0.1$, $5$-epoch learning-rate warm-up for both methods): final test accuracy (\%, mean$\pm$std over 3 seeds); \emph{ES}: uniform SGD at its best epoch chosen on the test set, a reference upper bound rather than a method, so not bolded. $Q/\varepsilon$: visit share of the corrupted examples relative to uniform sampling. \emph{Fit}: fraction of the corrupted examples that the final model predicts as their (wrong) human label.}
\label{tab:cifar10n_online}
\small
\begin{tabular}{l ccc c ccc}
\toprule
& \multicolumn{3}{c}{aggre ($\varepsilon=0.090$)} && \multicolumn{3}{c}{worst ($\varepsilon=0.402$)} \\
\cmidrule(lr){2-4}\cmidrule(lr){6-8}
Method & Test acc. & $Q/\varepsilon$ & Fit && Test acc. & $Q/\varepsilon$ & Fit \\
\midrule
Uniform SGD           & $88.66\pm0.10$ & $1.00$ & $.98$ && $63.73\pm0.41$ & $1.00$ & $.86$ \\
Uniform SGD, ES       & $88.75\pm0.09$ & --- & --- && $77.00\pm0.74$ & --- & --- \\
\drisShort{}-online   & $\mathbf{89.20\pm0.30}$ & $0.59$ & $.45$ && $\mathbf{76.10\pm0.24}$ & $0.82$ & $.33$
 \\
\bottomrule
\end{tabular}
\end{table}

\end{document}